\documentclass{article}

\usepackage[english]{babel}
\usepackage{fullpage}
\usepackage[utf8]{inputenc} % allow utf-8 input
\usepackage{lmodern}
\usepackage[T1]{fontenc}    % use 8-bit T1 fonts

\usepackage{amsfonts,array,url,booktabs,microtype,cite,multirow,subcaption,graphicx,mathtools,epstopdf,xspace,amsthm,amsfonts,nicefrac,xcolor,makecell,adjustbox,enumitem,wrapfig,lipsum,booktabs}
\usepackage[colorlinks=true, allcolors=blue]{hyperref}
\usepackage{wrapfig,lipsum,booktabs}
\usepackage{olo}
\newtheorem*{assumption*}{\assumptionnumber}
\newcommand{\alg}{{NGAME}\xspace}
\newcommand{\algtitle}{\alg: Negative Mining-aware Mini-batching for Extreme Classification}

\newcommand{\suppl}{appendices\xspace}
\newcommand{\gpu}{A100\xspace}
\newcommand{\bert}{BERT\xspace}

\newcommand{\ova}{1-vs-all\xspace}

\renewcommand{\norm}[1]{\left\lVert#1\right\rVert}
\definecolor{arsenic}{rgb}{0.23, 0.27, 0.29}

\definecolor{silver}{rgb}{0.75, 0.75, 0.75}

\newcommand\blfootnote[1]{%
  \begingroup
  \renewcommand\thefootnote{}\footnote{#1}%
  \addtocounter{footnote}{-1}%
  \endgroup
}

\title{\algtitle}
\author{Kunal Dahiya\thanks{Authors contributed equally} \thanks{IIT Delhi} , Nilesh Gupta\footnotemark[1] \thanks{UT Austin} , Deepak Saini\footnotemark[1] \thanks{Microsoft Research} , Akshay Soni\thanks{Microsoft} , Yajun Wang\footnotemark[5] ,\\Kushal Dave\footnotemark[5] , Jian Jiao\footnotemark[5] , Gururaj K\footnotemark[5] , Prasenjit Dey\footnotemark[5] , Amit Singh\footnotemark[5] , \\Deepesh Hada\footnotemark[4] , Vidit Jain\footnotemark[4] , Bhawna Paliwal\footnotemark[4] , Anshul Mittal\footnotemark[2] , Sonu Mehta\footnotemark[4] , \\Ramachandran Ramjee\footnotemark[4] , Sumeet Agarwal\footnotemark[2] , Purushottam Kar\thanks{IIT Kanpur} , Manik Varma\footnotemark[4]}

\begin{document}
\maketitle

\blfootnote{Author Emails: {kunalsdahiya@gmail.com,  nileshgupta2797@utexas.edu, \{desaini, akson, yajunw, kudave, jiajia, gururajk, prdey, siamit, deepeshhada, t-viditjain, bhawna, Sonu.Mehta, ramjee, manik\}@microsoft.com, me@anshulmittal.org, sumeet@ee.iitd.ac.in, purushot@cse.iitk.ac.in}}

\begin{abstract}
Extreme Classification (XC) seeks to tag data points with the most relevant subset of labels from an extremely large label set. Performing \emph{deep XC} with dense, learnt representations for data points and labels has attracted much attention due to its superiority over earlier XC methods that used sparse, hand-crafted features. Negative mining techniques have emerged as a critical component of all deep XC methods that allow them to scale to millions of labels. However, despite recent advances, training deep XC models with large encoder architectures such as transformers remains challenging. This paper identifies that memory overheads of popular negative mining techniques often force mini-batch sizes to remain small and slow training down. In response, this paper introduces \alg, a light-weight mini-batch creation technique that offers provably accurate in-batch negative samples. This allows training with larger mini-batches offering significantly faster convergence and higher accuracies than existing negative sampling techniques. \alg was found to be up to 16\% more accurate than state-of-the-art methods on a wide array of benchmark datasets for extreme classification, as well as 3\% more accurate at retrieving search engine queries in response to a user webpage visit to show personalized ads. In live A/B tests on a popular search engine, \alg yielded up to 23\% gains in click-through-rates.
\end{abstract}

\section{Introduction}
\label{sec:intro}

\textbf{Overview}: Extreme Classification (XC) requires predicting the most relevant {\it subset} of labels for a data point from an extremely large set of labels. Note that multi-label classification generalizes multi-class classification which instead aims to predict a single label from a mutually exclusive label set. In recent years, XC has emerged as a workhorse for several real-world applications such as product recommendation~\cite{Medini19,Dahiya21,Mittal21}, document tagging~\cite{Babbar17,You18,Chang20}, search \& advertisement~\cite{Prabhu18b,Dahiya21,Jain16}, and query recommendation~\cite{Jain19,Chang20}.

\textbf{Challenges in XC}: XC tasks give rise to peculiar challenges arising from application-specific demands as well as data characteristics. For instance, to adequately serve real-time applications such as ranking and recommendation, XC routines must offer millisecond-time inference, even if selecting among several millions of labels. Training models at extreme scales is similarly challenging due to the infeasibility of training on all data point-label pairs when the number of data points and labels both are in the millions. This necessitates the use of some form of \textit{negative mining} technique wherein a data point is trained only with respect to its positive labels (of which there are usually a few) and a small set of carefully chosen negative labels. Training is made even more challenging due to the abundance of \textit{rare} or \textit{tail} labels for which few, often less than 10, training points are available. It is common for a majority of labels to be rare in ranking, recommendation, and document tagging tasks.

\textbf{Deep Siamese XC}: The recent years have have seen XC models become more and more accurate with the development of two specific design choices. Firstly, XC methods identified the benefits of using label metadata in various forms such as textual descriptions of labels \cite{Chang20,Mittal21,Dahiya21b} or label correlation graphs \cite{Mittal21b,Saini21}. This is in sharp contrast to earlier work in XC that treated labels as featureless identifiers (please see Section~\ref{sec:related} for a survey). Secondly, XC methods began reaping the benefits of deep-learnt embeddings by using deep encoder architectures to embed both data points and labels. This was yet another departure from traditional XC methods that relied on sparse, hand-crafted features such as bag-of-words. In particular, recent work has demonstrated advantages of using \textit{Siamese} architectures wherein representations for both data points as well as labels are obtained using a shared embedding model \cite{Dahiya21b,Lu20,Mittal21}. Some of these techniques \cite{Dahiya21b,Mittal21} demonstrate further gains in accuracy by fine-tuning the label representation offered by the shared embedding architecture to yield label classifiers. Other methods such as \cite{Zhang21} focus on training large transformer models at extreme scales. Together, these methods constitute the state-of-the-art across a wide range of retrieval and recommendation tasks.

\textbf{\alg and Our Contributions}: Despite these advances, training deep XC models based on transformer encoders poses steep time and memory overheads (see Sec.~\ref{sec:method} for a detailed discussion). In particular, popular negative mining techniques often themselves pose memory overheads when used to train large encoder architectures like transformers. This forces mini-batch sizes to remain small leading to slow convergence and sub-optimal accuracies. As a remedy, existing methods either settle for simpler encoders such as bag-of-embeddings~\cite{Dahiya21b} or else forego the use of label text altogether in an attempt to speed up training ~\cite{Zhang21,Chalkidis19,Jiang21}, both of which result in sub-optimal accuracies. This paper develops the \alg method for training large transformer-based Siamese architectures on XC tasks at the scale of 85 million labels. Training is performed in two stages: a pre-training stage first learns an effective Siamese architecture, followed by a fine-tuning stage that freezes the Siamese embeddings and learns a per-label refinement vector to fine-tune the label embeddings to obtain the final classifiers. To summarize, this paper makes the following contributions:
\begin{enumerate}[wide, labelwidth=!, labelindent=0pt]
    \item It notices that existing techniques treat mini-batching and negative mining as unrelated problems and this leads to inefficient training. In response, the \alg method proposes a light-weight negative mining-aware technique for creating mini-batches. A curious property of the mini-batches so obtained is that in-batch sampling itself starts offering informative negative samples and faster convergence than ANNS-based negative mining techniques such as ANCE (see Section~\ref{sec:related}). This leads to much more efficient training as in-batch sampling has much smaller time and memory overheads than ANNS-based methods.
    \item The above allows \alg to operate with much larger mini-batch sizes and allows training with 85 million labels on multiple GPUs without the need of a specialized algorithm to reduce inter-GPU communication. This becomes particularly critical when training large transformer-based feature architectures such as \bert which already restrict mini-batch sizes owing to their large memory footprint.
    \item Theoretical analysis is performed where \alg is shown to offer provably-accurate negative samples with bounded error under fairly general conditions. Furthermore, under some standard, simplifying assumptions, convergence to a first-order stationary point is also assured.
    \item \alg is shown to offer 16\% more accurate predictions than state-of-the-art methods including \bert-based methods such as LightXML~\cite{Chang19}, XR-Transformers~\cite{Zhang21} on a wide array of benchmark datasets. Moreover, in live A/B tests on a popular search engine, \alg yielded up to 23\% gains in click-through-rates over an ensemble of state-of-the-art generative, XC, information retrieval (IR) and two-tower models.
\end{enumerate}
\section{Related Work}
\label{sec:related}

\textbf{Extreme Classification (XC)}: Early extreme classification algorithms focused primarily on designing accurate and scalable classifiers for either sparse bag-of-words~\cite{Agrawal13,Weston13,Mineiro15,Babbar17,Babbar19,Bhatia15,Jasinska16,Khandagale19,Jain16,Prabhu14,Prabhu18,Prabhu18b,Tagami17,Yen16,Yen17,Prabhu20,Xu16,Wei19,Siblini18a,Niculescu17,Panda19,Cisse13,Barezi19} or dense pre-trained features obtained from feature extraction models such as CDSSM~\cite{Huang13,Jain19} or fastText~\cite{Joulin17}. More recent work demonstrated that learning task-specific features can lead to significant gains in classification accuracy. The use of diverse feature extraction architectures including Bag-of-embeddings~\cite{Dahiya21,Guo19,Gupta19,Wydmuch18,Zhang18,Guo19,Medini19}, CNN~\cite{Liu17, Kharbanda21,Chen22}, LSTM~\cite{You18}, and \bert~\cite{Devlin19,Chang20,Jiang21,Chalkidis19,Ye20} led to the development of \textit{deep extreme classification} (deep XC) techniques that offered significant gains over techniques that either used Bag-of-Words or pre-trained features. However, these techniques did not make any use of label metadata and continued to treat labels as indices devoid of features. Still more recent work~\cite{Chang20,Dahiya21,Mittal21,Mittal21b,Saini21} introduced various techniques to exploit label metadata such as textual descriptions of labels or label correlation graphs. For instance, the X-Transformer~\cite{Chang20} and XR-Transformer~\cite{Zhang21} methods make use of label text to clusters labels in their \textit{Semantic Label Indexing} phase. However, these methods employ an ensemble making them expensive to scale to large datasets. Other techniques use label metadata to inform classifier learning more intimately via label text~\cite{Mittal21,Dahiya21b}, images~\cite{Mittal22}, or label correlation graphs~\cite{Mittal21b,Saini21}. In all cases, careful use of label metadata is shown to offer better classification accuracies, especially on data-scarce tail labels.

\textbf{Siamese methods for XC}: Although widely studied in areas such as information retrieval \cite{Xiong20} and computer vision \cite{He20}, Siamese methods received attention in XC more recently due to their ability to incorporate label metadata in diverse forms. The key goal in Siamese classification is to learn embeddings of labels and data points such that data points and their positive labels are embedded in close proximity whereas data points and their negative labels maintain a discernible separation. DECAF~\cite{Mittal21} and ECLARE~\cite{Mittal21b} use asymmetric networks to embed data points and labels whereas SiameseXML~\cite{Dahiya21b} and GalaXC~\cite{Saini21} use a symmetric network. It is notable that with respect to the type of metadata used, DECAF and SiameseXML use only label text whereas ECLARE and GalaXC use label graphs. DECAF and ECLARE struggle to scale to tasks with over a million labels due to an expensive shortlisting step that seems critical for good performance. On the other hand, GalaXC requires pre-learnt embeddings for data points and labels from models such as Astec~\cite{Dahiya21} that were trained on the same task. This combined with the execution of a multi-layer graph convolutional network makes the method more expensive. Of special interest to this paper is the SiameseXML technique that yields state-of-the-art accuracies on short-text datasets and can scale to 100M labels. However, to achieve such scales SiameseXML restricts itself to a low-capacity Bag-of-embeddings feature architecture and uses expensive but suboptimal negative sampling techniques discussed below.

\textbf{Negative Mining}: Training on a small but carefully selected set of irrelevant labels per training point is critical for accurately scaling XC algorithms. This is because training on all irrelevant labels for every data point becomes infeasible when the number of training points and labels are both in the millions. Negative mining techniques come in three flavours: 

1. Oblivious: these include random sampling wherein negative samples are either drawn uniformly or from token/unigram distributions~\cite{Mikolov13}, and in-batch sampling~\cite{Dahiya21b,Guo19,Faghri18,Chen20,He20} wherein negative samples are identified among positive samples of other data points within the same mini-batch. These are computationally inexpensive but can offer uninformative negatives and slow convergence~\cite{Xiong20} (please see Fig.~\ref{fig:convergence}).

2. Feature-aware: these techniques identify \textit{hard} negatives based on either sparse raw features such as BM25~\cite{Lee18,Luan20} or features from a pre-trained model like \bert~\cite{hofstatter21}. However, since these features are not necessarily aligned to the task, the negative samples may offer biased training.

3. Task-aware: methods such as ANCE~\cite{Xiong20} and SiameseXML~\cite{Dahiya21b} retrieve hard negatives using features learnt for the task at hand. Since these features keep getting updated during training, an Approximate Nearest Neighbor Search (ANNS) index has to be recomputed every few epochs. Similarly, RocketQA~\cite{Qu21} trains an expensive cross-encoder model to eliminate false negatives. Although these methods offer better negatives, they also add substantial time and memory overhead {\it e.g.,} ANCE~\cite{Xiong20} recomputes an ANNS index over label embeddings every few training epochs which can be expensive even with asynchronous training. Moreover, these strategies are also more memory intensive as discussed in Section~\ref{sec:method}.

In contrast, \alg provides a strategy for mining provably-accurate negatives similar to task-aware methods (see Theorem~\ref{thm:main}) but with overheads comparable to those of oblivious techniques. On multiple datasets, the overhead of executing \alg's negative mining technique was up to 1\%  as compared to 210\% for ANCE (see Tab.~\ref{tab:ablation_speedup}) allowing \alg to train with powerful feature architectures such as \bert at extreme scales and offer accuracies superior to leading oblivious, feature- and task-aware negative mining techniques.

\section{\alg: \underline{N}e\underline{G}ative Mining-\underline{A}ware \underline{M}ini-batching for  \underline{E}xtreme Classification}
\label{sec:method}

\begin{figure*}[t]
    \centering
    \includegraphics[width=0.95\textwidth]{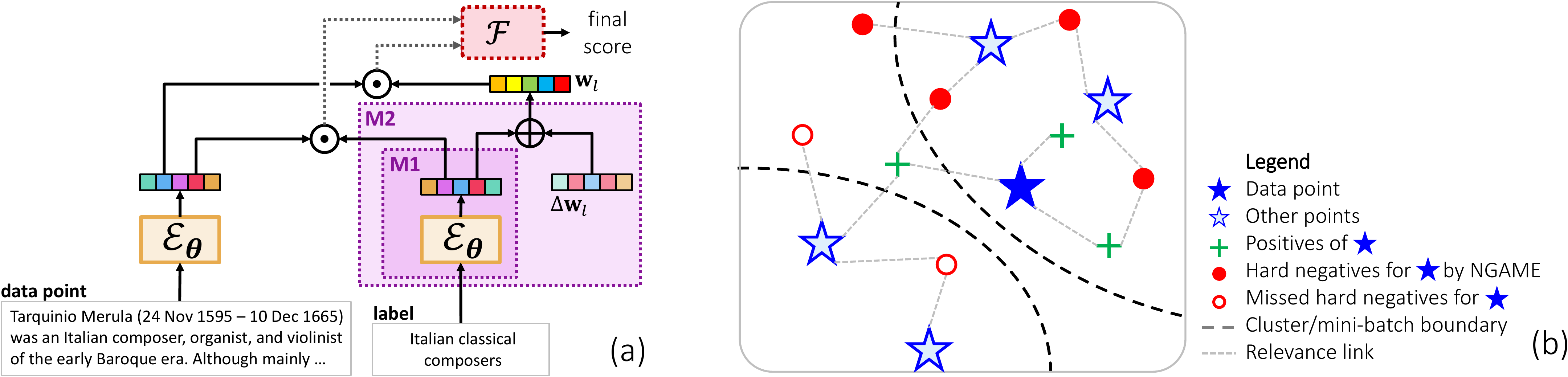}
    \caption{\small (Left) \alg's model architecture. Data points and labels are embedded using a Siamese encoder model and per-label \ova classifiers are learnt. Classifier and Siamese scores are combined to offer final output. (Right) A depiction of \alg's negative mining strategy. A light dotted line indicates a data point-relevant label pair. \alg misses only those hard negative labels for a data point that are not relevant to any other data point in its cluster. See also Figure~\ref{fig:negatives} for a real example illustrating the hard-negatives retrieved by \alg.}
    \label{fig:arch}
\end{figure*}

\textbf{Notation}: $L$ is the total number of labels (note that the same set of labels is used during training and testing). $\vx_i, \vz_l \in \cX$ denote textual representations of data point $i$ and label $l$ respectively. The training set $D:=\{\{\vx_i,\vy_i\}_{i=1}^{N},\{{\vz_l}\}_{l=1}^{L}\}$ consists of $L$ labels and $N$ data points. For a data point $i \in [N]$, $\vy_i \in\{-1,+1\}^L$ is its ground truth vector, {\it i.e.}, $y_{il} = +1$ if label $l$ is relevant/positive for the data point $i$ and $y_{il} = -1$ otherwise.

\textbf{Architecture}: \alg uses an encoder model (DistilBERT base~\cite{Sanh2019DistilBERTAD,Wolf20}) denoted by $\cE_{\vtheta}: \cX \rightarrow \cS^{D-1}$, with trainable parameters $\vtheta$, to embed both data point and label text onto the $D$-dimensional unit sphere $\cS^{D-1}$. \alg also uses \ova classifiers $\vW \deff \{\vw_l\}_{l \in [L]}$ where $\vw_l$ is the classifier for label $l$.

\textbf{Inference Pipeline}: After training $\cE_{\vtheta},\vW$, \alg establishes a maximum inner product search (MIPS)~\cite{MalkovY16} structure over the set of learnt label classifiers $\{\vw_l\}_{l \in [L]}$. Given a test data point $\vx_t$ its embedding $\cE_{\vtheta}(\vx_t)$ is used to invoke the MIPS structure that returns labels $l \in [L]$ with highest classifier scores $\vw_l^\top\cE_{\vtheta}(\vx_t)$. Mild gains were observed by fusing classifier scores with Siamese scores $\cE_{\vtheta}(\vz_l)^\top\cE_{\vtheta}(\vx_t)$ for the shortlisted labels using a simple fusion architecture $\cF$ and recommending labels in decreasing order of the fused scores (see Appendix~\ref{app:implementation} for details). Appendix~\ref{app:complexity} establishes that \alg offers $\bigO{b + D\log L}$ time inference.

\textbf{Joint Training and its Limitations}: The encoder parameters $\vtheta$ and \ova classifiers $\vW = \bc{\vw_l}_{i=1}^{L}$ could be trained by minimizing a triplet loss function
\begin{equation}
    \label{eq:loss}
    \min_{\vtheta, \vW} \cL(\vtheta, \vW) = 
 \sum_{i = 1}^N\sum_{\substack{l: y_{il} = +1\\k: y_{ik} = -1}} [\vw_k^\top\cE_{\vtheta}(\vx_i) - \vw_l^\top\cE_{\vtheta}(\vx_i) + \gamma]_+.
\end{equation}

In the above expression, the indices $l$ and $k$ run over the relevant and irrelevant labels for data point $i$ respectively with $\vw_l,\vw_k$ being their label-wise classifiers. Consequently, \eqref{eq:loss} encourages every relevant label to get a score at least margin $\gamma$ higher than every irrelevant label. Other task losses such as BCE or contrastive could also have been used.

For an encoder model with $b$ parameters, a training epoch w.r.t. $\cL(\vtheta, \vW)$ would take $\Om{NL(b+D)\log L} = \Om{NL}$ time (since data points usually have $\bigO{\log L}$ relevant labels -- see Tab.~\ref{tab:data_stats}) which is prohibitive when $N$ and $L$ are both in the millions. Joint training also requires storing both the encoder and classifier models in memory. Along with overheads of optimizers such as Adam, this forces mini-batches to be small and slows down convergence \cite{Qu21}. Common workarounds such as distributed training on a GPU cluster~\cite{Song20} impose overheads such as inter-GPU communication. \alg instead proposes a novel two-pronged solution based on modular training and negative mining-aware mini-batching.

\textbf{Modular Training}: Recent papers~\cite{Dahiya21b,Mittal21b} have postulated that label embeddings can serve as surrogates for label classifiers. To effect this intuition, \alg reparameterizes label classifiers as $\vw_l = \cE_{\vtheta}(\vz_l) + \veta_l$, where $\cE_{\vtheta}(\vz_l)$ is the label embedding and $\veta_l$ is a free $D$-dimensional \textit{residual} vector. In \textbf{module M1}, \alg sets the residual vectors to zero, {\it i.e.,} using $\vw_l = \cE_{\vtheta}(\vz_l)$. Plugging this into \eqref{eq:loss} results in a loss expression that depends only on $\vtheta$ and encourages label-data point embedding similarity $\cE_{\vtheta}(\vx_i)^{\top} \cE_{\vtheta}(\vz_l)$ to be high when $y_{il} =+1$ and low otherwise. Note that this is identical to Siamese training for encoder models~\cite{Dahiya21b,Guo19,Faghri18,Chen20,He20,Xiong20,hofstatter21} and allows \alg to perform supervised training of the encoder model $\vtheta$ by itself in M1. In \textbf{module M2}, the learnt encoder model $\hat\vtheta$ is frozen and residual vectors $\veta_l$ are implicitly learnt by initializing $\vw_l = \cE_{\hat\vtheta}(\vz_l)$ and minimizing \eqref{eq:loss} with respect to $\vW$ alone. Appendix~\ref{app:implementation} gives details of M1 and M2 implementation.

\textbf{Benefits of Modular Training}: Modular training generalizes the Siamese encoder training strategy popular in literature and makes more effective use of label text during training that benefits rare labels with scant supervision. Splitting encoder and classifier training also eases memory overheads in each module -- the parameters and gradients of the encoder and classifier are not required to be stored (in the GPU memory) or computed at the same time. Please refer to DeepXML~\cite{Dahiya21} that discusses the benefits of modular training in great depth. Combined with memory savings offered by \alg's negative sampling technique (discussed below), this enables \alg to train with 1.5$\times$ larger mini-batches compared to approaches such as ANCE~\cite{Xiong20} thus offering faster convergence. For instance, \alg could be trained within 50 hours on 6$\times$\gpu GPUs on the PR-85M XC task with more than 85 million labels.

\textbf{Negative Mining and its Overheads}: Modular training partly lowers the memory costs of training but not its $\Om{NL}$ computational cost. To reduce computational costs to $\bigO{N\log L}$, \textit{negative mining} techniques are critical. These restrict training to only the positive/relevant labels for a data point (which are usually only $\bigO{\log L}$ many -- see Tab.~\ref{tab:data_stats}) and a small set of the hardest-to-classify negative/irrelevant labels for that data point. However, such techniques end up adding their own memory and computational overheads. For example, ANCE~\cite{Xiong20} queries an ANNS structure over label embeddings that needs to be periodically refreshed (since label embeddings get trained in parallel) while RocketQA~\cite{Qu21} trains a cross-encoder to weed out false negatives; these add significant overhead to training. Computing label embeddings $\cE_{\theta}(\vz_l)$ for retrieved hard negatives is an additional overhead faced by all these methods. For a mini-batch of, say $S$ training data points, if a single positive and $H$ hard-negatives are chosen per data point, then a total of $S(H + 2)$ embeddings need to be computed. For large transformer-based encoder models, this forces batch sizes $S$ to be very small. A notable exception is in-batch negative mining~\cite{Dahiya21b,Guo19,Faghri18,Chen20,He20,Karpukhin20} that looks for hard negatives of a data point only among positives of other data points in the same mini-batch and requires computing only $2S$ embeddings instead of $2S + HS$ embeddings. However, this technique may offer uninformative negatives and slow convergence~\cite{Xiong20} if mini-batches are created randomly.

\renewcommand{\algorithmiccomment}[1]{{\hfill//\emph{#1}}}

\begin{algorithm}[t]
	\caption{\alg's negative mining strategy}
	\label{algo:iteration}
	{\small
	\begin{algorithmic}[1]
		\REQUIRE Init model $\vtheta^0$, mini-batch size $S$, cluster size $C$, refresh interval $\tau$, hardness threshold $r$%
% 		\ENSURE An oracle for mini-batching and negative mining
		\FOR{$t = 0, \ldots, T$}
		\IF[Redo clustering at regular intervals]{$t \% \tau = 0$}
		    \STATE Cluster current data point embeddings $\cE_{\vtheta^t}(\vx_i)$ 
		    into $\ceil{N/C}$ clusters, with each cluster containing $\approx C$ data points each.
		\ENDIF
		\STATE Choose $\ceil{S/C}$ random clusters to create a mini-batch $S_t$ of size $S$
		\STATE Take positive labels  $\cP^i_+$ for each data point $i \in S_t$ and let $L_t \deff \bigcup_{i\in S_t}\cP^i_+$%
		\STATE Compute $\cE_{\vtheta^{t-1}}(\vx_i),\cE_{\vtheta^{t-1}}(\vz_l)$ for all $i \in S_t, l \in L_t$.
		\STATE Select hard-negatives for data point $i \in S_t$ as $\bc{l \in L_t: \norm{\cE_{\vtheta}(\vx_i) - \cE_{\vtheta}(\vz_l)}_2 \leq r}$ 
		\STATE Update $\vtheta^t$ using mini-batch SGD over $S_t$
		\ENDFOR
	\end{algorithmic}
	}
\end{algorithm}

\textbf{Negative Mining-aware Mini-batching}: \alg notes that the inexpensive in-batch sampling technique could have offered good-quality negatives had mini-batch creation encouraged hard-negatives of a data point to exist among positives of other data points in the mini-batch. Since Siamese training encourages embeddings of data points and related labels to be close {\it i.e.,} $\norm{\cE_{\vtheta}(\vx_i) - \cE_{\vtheta}(\vz_l)}_2 \ll 1$ if $y_{il} = 1$, if mini-batches are created out of data points that lie close to each other, triangle inequality ensures that in-batch negatives are of good quality. Figure~\ref{fig:arch} depicts this pictorially and Theorem~\ref{thm:main} assures this formally. The resulting mini-batching and negative mining strategy is presented in Algorithm~\ref{algo:iteration} and offers provably accurate hard-negatives at the low computational cost of in-batch negative mining. It also imposes very little compute overhead, requiring occasional data clustering that is inexpensive compared to ANNS creation or cross-encoder training. On a dataset with 85 million labels and 240 million training data points, refreshing ANCE's ANNS index and hard-negative retrieval for all training data points took 4 hours whereas \alg took under an hour. ANCE needed to compute embeddings for hard negatives separately causing its GPU memory requirement for label embeddings to be at least $2\times$ higher than \alg that only needed to compute embeddings for positive labels (since \alg's negatives are sampled from the positives of other data points in the mini-batch). Thus, \alg imposed a mere 1\% increase in epoch time over vanilla in-batch sampling whereas the same was up to 210\% for ANCE~(see Table~\ref{tab:ablation_speedup}). \alg also significantly outperformed techniques such as TAS~\cite{hofstatter21} that use task-agnostic negatives obtained by clustering static features (see Fig.~\ref{fig:convergence}). This is because previous works use static one-time clustering~\cite{Zhang21,hofstatter21} using features that are not aligned to the given task and offer poor quality negatives and biased training as a result. On the other hand, \alg continuously adapts its negatives to the progress made by the training algorithm, thus offering provably accurate negatives and convergent training (see Theorem~\ref{thm:main}). Moreover, \alg accomplishes this with much smaller overheads compared to existing state-of-the-art approaches (see Table~\ref{tab:ablation_speedup} in \suppl).

\textbf{Curriculum Learning with \alg}: The use of extremely hard negatives may impact training stability in initial epochs when the model is not well-trained~\cite{Harwood17}. Thus, it is desirable to initiate training with easier negatives. Fortunately, \alg can tweak the ``hardness'' of its negatives by simply changing the cluster size $C$ in Algorithm~\ref{algo:iteration}. Using $C = 1$ makes \alg identical to vanilla in-batch negative mining which offers easier negatives whereas the other extreme $C \rightarrow N$ causes \alg to behave identical to ANNS-based techniques such as ANCE~\cite{Xiong20} that offer the hardest-to-classify negatives. This allows \alg to commence training with small values of $C$ and gradually increase it to get progressively harder negatives in later epochs. This curriculum learning approach could lead to $25\%-40\%$ faster convergence when compared to training that used a fixed cluster size $C$.

\section{Theoretical Analysis}
\label{sec:analysis}
Theorem~\ref{thm:main} guarantees that negative samples offered by \alg are provably accurate. The key behind this result is the intuition that \alg's negative mining strategy should succeed if embeddings of data points and relevant labels lie in close proximity and clustering is sufficiently precise. To state this result, we define the notion of a good embedding and clustering.
\begin{definition}[$(r,\epsilon)$-good embedding]
\label{defn:good-embed}
An encoder model $\cE$ parameterized by $\vtheta$ is said to be $(r,\epsilon)$-good if
\[
\Pp{i \in [N], l: y_{il} = 1}{\norm{\cE_{\vtheta}(\vx_i) - \cE_{\vtheta}(\vz_l)}_2 > r} \leq \epsilon.
\]
\end{definition}
\begin{definition}[$(r,\epsilon)$-good clustering]
\label{defn:good-clust}
A clustering of a given a set of data point embedding vectors $\cE_{\vtheta}(\vx_i)$ into $K = \ceil{N/C}$ clusters $C_1, \ldots, C_K$ is said to be $(r,\epsilon)$-good if
\[
\Pp{i, j}{\norm{\cE_{\vtheta}(\vx_i) - \cE_{\vtheta}(\vx_j)}_2 \leq r, c(i) \neq c(j)} \leq \epsilon,
\]
where $c(i) \in [K]$ tells us the cluster to which data point $i$ as assigned.
\end{definition}
\begin{theorem}[Negative Mining Guarantee]
\label{thm:main}
Suppose Algorithm~\ref{algo:iteration} performs its clustering step using data point embeddings obtained using the encoder model $\cE_{\vtheta}$ parameterized by $\vtheta$ and identifies a set of hard negative labels $\hat\cP^i_-$ for data point $i$ in step 8 of the algorithm using the threshold $r > 0$. For any $i \in [N], l \in [L]$, let $E^r_{il} := \bc{y_{il} = -1} \wedge \bc{\norm{\cE_{\vtheta}(\vx_i) - \cE_{\vtheta}(\vz_l)}_2 \leq r} \wedge \bc{l \notin \hat\cP^i_-}$ be the \textit{bad} event where $l$ is an $r$-hard negative label for data point $i$ but \alg fails to retrieve it. Then if the model $\vtheta$ was $(r,\epsilon_1)$-good and the clustering was $(2r,\epsilon_2)$-good, we are assured that
\[
\frac1{NL}\sum_{i \in [N], l \in [L]}\bI\bs{E^r_{il}} \leq c_1\cdot\epsilon_1 + c_2\cdot\epsilon_2,
\]
for some constants $c_1,c_2$ that are entirely independent of the execution of the \alg algorithm and depend only on certain dataset statistics (such as number of data points per label etc).
\end{theorem}
A description of the constants $c_1,c_2$ as well as the complete proof are provided in Appendix~\ref{supp:theory}. However, to appreciate the essence of the guarantees, it helps to look at a special case of the result as indicated in the following corollary:
\begin{corollary}
\label{cor:label-informal}
For the special case when all data points have the same number of relevant labels and all labels are relevant to the same number of data points, we have $c_1, c_2 \leq 1$ which gives us
\[
\frac1{NL}\sum_{i \in [N], l \in [L]}\ind{E^r_{il}} \leq \epsilon_1 + \epsilon_2
\]
\end{corollary}

Upon making certain simplifying assumptions, \alg is also able to guarantee convergence to an approximate first-order stationary point. For sake of simplicity, this result is presented for Module M1 (Encoder Training) but a similar result holds for Module M2 (Classifier Training) as well.

\begin{theorem}[First-order Convergence Guarantee]
\label{thm:conv-informal}
Suppose \alg is executed with full-batch gradient updates with eager clustering i.e. $\tau = 1$ in Algorithm~\ref{algo:iteration}. Also suppose the loss function and architecture are smooth and offer bounded gradients. Also, let $\epsilon_{\text{tot}}^t \deff c_1\epsilon_1^t + c_2\epsilon_2^t$ denote the total error assured by Theorem~\ref{thm:main-full} in terms of hard-negative terms missed by \alg at iteration $t$. Then there exists a smoothed objective $\tilde\cL$ such that within $T$ iterations, \alg arrives at a parameter $\vtheta$ that either satisfies
\[
\norm{\nabla\tilde\cL(\vtheta)}_2 \leq \bigO{\frac1T}
\]
or else for some $t \leq T$, it satisfies
\[
\norm{\nabla\tilde\cL(\vtheta)}_2 \leq \bigO{\epsilon_{\text{tot}}^t}.
\]
\end{theorem}
We note that some of the assumptions (full batch descent, eager clustering etc) are not essential to this result and merely simplify the proof whereas other assumptions (smoothness, bounded gradients) are standard in literature. The exact statement of the assumptions, a description of the smoothed objective $\tilde\cL$, the description of the iterate $t$ for which $\epsilon^t_{\text{tot}}$ bounds the error in the second case, and the proof, are all available in Appendix~\ref{supp:theory}.

\section{Experiments} \label{sec:results}
\begin{table}
	\caption{\small Dataset Statistics. For all datasets, $L = \Theta(N)$ and data points typically have $\bigO{\log L}$ relevant labels. The public datasets can be downloaded from The Extreme Classification Repository~\cite{XMLRepo}. %A $\ddagger$ sign denotes information that was redacted for proprietary datasets
	}
	\label{tab:data_stats}
	    \centering
	\resizebox{0.8\linewidth}{!}{
		\begin{tabular}{l|ccccc}
			\toprule
			\textbf{Dataset} & 
			\makecell{\textbf{Train}\\$N$} & 
			\makecell{\textbf{Labels}\\$L$}  & 
			\makecell{\textbf{Test}\\ $N'$} & 
			\makecell{\textbf{Avg. data points}\\\textbf{per label}} &
			\makecell{\textbf{Avg. labels}\\\textbf{per data point}} \\
			\midrule
			\multicolumn{6}{c}{Short-text benchmark datasets}\\ \midrule
            LF-AmazonTitles-131K & 294,805 & 131,073 & 134,835 & 2.29 & 5.15 \\
            % LF-WikiSeeAlsoTitles-320K & 693,082 & 312,330 & 177,515 & 2.11 & 4.68 \\
            LF-AmazonTitles-1.3M & 2,248,619 & 1,305,265 &  970,237 &  22.20 & 38.24 \\
            \midrule
			\multicolumn{6}{c}{Full-text benchmark datasets}\\ \midrule
            LF-Amazon-131K & 294,805 & 131,073 & 134,835 & 2.29 & 5.15  \\
            LF-WikiSeeAlso-320K & 693,082 & 312,330 & 177,515 & 2.11 & 4.68 \\
            LF-Wikipedia-500K & 1,813,391 & 501,070 & 783,743 & 4.77 & 24.75 \\
            \bottomrule
    	\end{tabular}
	}
\end{table}

\textbf{Datasets}: Multiple short-text as well as full-text benchmark datasets were considered in this paper which can be downloaded from the Extreme Classification Repository~\cite{XMLRepo}. Both title and detailed content were available for full-text datasets~(LF-Amazon-131K, LF-WikiSeeAlso-320K and LF-Wikipedia-500K) whereas only the product/webpage titles were available for the short-text datasets~(LF-AmazonTitles-131K and LF-AmazonTitles-1.3M). These datasets cover a variety of applications including product-to-product recommendation~(LF-Amazon-131K, LF-AmazonTitles-131K, and LF-AmazonTitles-1.3M), predicting related Wikipedia pages~(LF-WikiSeeAlso-320K) and predicting Wikipedia categories~(LF-Wikipedia-500K). Please refer to Table~\ref{tab:data_stats} for data statistics.

\textbf{Baselines}: Siamese methods for XC such as SiameseXML~\cite{Dahiya21b}, DECAF~\cite{Mittal21}, and ECLARE~\cite{Mittal21b} are the main baselines for \alg. Other significant baselines include non-Siamese deep XC methods such as XR-Transformer~\cite{Zhang21}, LightXML~\cite{Jiang21}, BERTXML~\cite{Chalkidis19}, MACH~\cite{Medini19}, AttentionXML~\cite{You18}, X-Transformer~\cite{Chang20} and Astec~\cite{Dahiya21}. Note that these include methods such as XR-Transformer~\cite{Zhang21}, BERTXML~\cite{Chalkidis19} and LightXML~\cite{Jiang21} that also rely on transformer encoders albeit those that are not trained in a Siamese fashion.
%Unfortunately, X-Transformer can be expensive as it requires a large array of GPUs even on moderate sized datasets. So, results for X-Transformer~\cite{Chang20} are reported as per the XML Repository.
For sake of completeness, results are also reported for classical XC methods Bonsai~\cite{Khandagale19}, DiSMEC~\cite{Babbar17}, Parabel~\cite{Prabhu18b}, XT~\cite{Wydmuch18}, and AnnexML~\cite{Tagami17}  (please refer to Section~\ref{app:results} in the \suppl for all results). Implementations provided by the respective authors were used for all methods. 

\textbf{Hyper-parameters}: \alg's hyper-parameters are: (i) cluster size, (ii) batch size, (iii) interval $\tau$ between clustering updates and (iv) margin value $\gamma$. The Adam optimizer was used to learn model parameters and its hyper-parameters include learning rate and number of epochs. \alg did not require much hyper-parameter tuning and default values were used for all hyper-parameters except for number of epochs. Please see Section~\ref{app:implementation} in \suppl for a detailed discussion on hyper-parameters. \alg's encoder $\cE_{\vtheta}$ was initialized with the 6-layered DistilBERT base~\cite{Sanh2019DistilBERTAD, Wolf20} to encode data points and labels. The hyper-parameters of baseline algorithms were set as suggested by their authors wherever applicable and by fine-grained validation otherwise.
  
\begin{table}
    \caption{\small Results on full-text benchmark datasets. See the \suppl for full results. TT refers to training time in hours on a single Nvidia V100 GPU.}
    \label{tab:results_repo_full_text}
      \centering
      \resizebox{\linewidth}{!}{
        \begin{tabular}{@{}l|ccccccccccc@{}}
        \toprule
        \textbf{Method} & \textbf{P@1} & \textbf{P@3} & \textbf{P@5} & \textbf{N@3} & \textbf{N@5} & \textbf{PSP@1} & \textbf{PSP@3} & \textbf{PSP@5} & \textbf{PSN@3} & \textbf{PSN@5} & \textbf{TT} \\ \midrule

        \multicolumn{12}{c}{LF-Wikipedia-500K}\\ \midrule
        \alg & \textbf{84.01} & \textbf{64.69} & \textbf{49.97} & \textbf{78.25} & \textbf{75.97} & \textbf{41.25} & \textbf{52.57} & \textbf{57.04} & \textbf{51.58} & \textbf{56.11} & 54.88  \\ 
        DPR & 79.91 & 59.51 & 45.9 & 72.69 & 70.58 & 37.57 & 46.51 & 50.7 & 45.96 & 50.16 & 54.67 \\ 
        TAS & 82.23 & 62.7 & 48.36 & 75.9 & 73.5 & 38.43 & 48.38 & 52.83 & 47.59 & 51.9 & 54.68 \\ 
        ANCE & 76.9 & 57.64 & 45.1 & 70.61 & 69.39 & 37.75 & 44.65 & 48.85 & 45.08 & 49.65  & 75.08 \\ 
        SiameseXML & 67.26 & 44.82 & 33.73 & 56.64 & 54.29 & 33.95 & 35.46 & 37.07 & 36.58 & 38.93 & 4.37 \\ 
        ECLARE & 68.04 & 46.44 & 35.74 & 58.15 & 56.37 & 31.02 & 35.39 & 38.29 & 35.66 & 38.72 & 9.4 \\ 
        DECAF & 73.96 & 54.17 & 42.43 & 66.31 & 64.81 & 32.13 & 40.13 & 44.59 & 39.57 & 43.7 & 13.4 \\ 
        XR-Transformer & 81.62 & 61.38 & 47.85 & 74.46 & 72.43 & 33.58 & 42.97 & 47.81 & 42.21 & 46.61 & 119.47 \\ 
        LightXML & 81.59 & 61.78 & 47.64 & 74.73 & 72.23 & 31.99 & 42 & 46.53 & 40.99 & 45.18 & 249 \\ 
        Astec & 73.02 & 52.02 & 40.53 & 64.1 & 62.32 & 30.69 & 36.48 & 40.38 & 36.33 & 39.84 & 6.39 \\ 
        Bonsai & 69.2 & 49.8 & 38.8 & 60.99 & 59.16 & 27.46 & 32.25 & 35.48 & $-$ & $-$ & 1.39 \\
        \midrule
        
        \multicolumn{12}{c}{LF-WikiSeeAlso-320K}\\ \midrule
        \alg & \textbf{47.65} & \textbf{31.56} & \textbf{23.68} & \textbf{47.5} & \textbf{48.99} & \textbf{33.83} & \textbf{37.79} & \textbf{41.03} & \textbf{38.36} & \textbf{41.01} & 75.39 \\ 
        SiamseXML & 42.16 & 28.14 & 21.39 & 41.79 & 43.36 & 29.02 &	32.68 & 36.03 & 32.64 & 35.17 & 2.33 \\ 
        ECLARE & 40.58 & 26.86 & 20.14 & 40.05 & 41.23 & 26.04 & 30.09 & 33.01 & 30.06 & 32.32 & 9.4 \\ 
        DECAF & 41.36 & 28.04 & 21.38 & 41.55 & 43.32 & 25.72 & 30.93 & 34.89 & 30.69 & 33.69 & 13.4 \\ 
        XR-Transformers & 42.57 & 28.24 & 21.3 & 41.99 & 43.44 & 25.18 & 30.13 & 33.79 & 29.84 & 32.59 & 119.47 \\ 
        LightXML & 34.5 & 22.31 & 16.83 & 33.21 & 34.24 & 17.85 & 21.26 & 24.16 & 20.81 & 22.8 & 249 \\ 
        BERTXML & 42.63 & 27.65 & 20.41 & 41.8 & 42.88 & 26.16 & 31.41 & 34.63 & 31.2 & 33.8 & 116.67 \\
        Astec & 40.07 & 26.69 & 20.36 & 39.36 & 40.88 & 23.41 & 28.08 & 31.92 & 27.48 & 30.17 & 6.39 \\ 
        Bonsai & 34.86 & 23.21 & 17.66 & 34.09 & 35.32 & 18.19 & 22.35 & 25.66 & 21.62 & 23.84 & 1.39 \\ 

        \midrule
        
        \multicolumn{12}{c}{LF-Amazon-131K}\\ \midrule
        \alg & \textbf{46.53} & \textbf{30.89} & \textbf{22.02} & \textbf{47.44} & 49.58 & \textbf{38.53} & \textbf{44.95} & \textbf{50.45} & \textbf{43.07} & \textbf{45.81} & 39.99\\
        SiameseXML & 44.81 & 30.19 & 21.94 & 46.15 & 48.76 & 37.56 &	43.69 & 49.75 & 41.91 & 44.97 & 1.18 \\ 
        ECLARE & 43.56 & 29.65 & 21.57 & 45.24 & 47.82 & 34.98 & 42.38 & 48.53 & 40.3 & 43.37 & 2.15 \\ 
        DECAF & 42.94 & 28.79 & 21 & 44.25 & 46.84 & 34.52 & 41.14 & 47.33 & 39.35 & 42.48 & 1.8 \\ 
        XR-Transformer & 45.61 & \textbf{30.85} & 22.32 & 47.1 & \textbf{49.65} & 34.93 & 42.83 & 49.24 & 40.67 & 43.91 & 38.4 \\ 
        LightXML & 41.49 & 28.32 & 20.75 & 42.7 & 45.23 & 30.27 & 37.71 & 44.1 & 35.2 & 38.28 & 56.03 \\ 
        BERTXML & 42.59	& 28.39 & 20.27	& 43.57 & 45.61 & 33.55 & 40.83 & 46.4 & 38.8 & 41.61 & 48.11 \\
        Astec & 42.22 & 28.62 & 20.85 & 43.57 & 46.06 & 32.95 & 39.42 & 45.3 & 37.45 & 40.35 & 3.05 \\ 
        Bonsai & 40.23 & 27.29 & 19.87 & 41.46 & 43.84 & 29.6 & 36.52 & 42.39 & 34.43 & 37.34 & 0.4 \\ 
        MACH & 34.52 & 23.39 & 17 & 35.53 & 37.51 & 25.27 & 30.71 & 35.42 & 29.02 & 31.33 & 13.91 \\ 

        \bottomrule

    \end{tabular}
    }
\end{table}

\begin{table}
    \caption{\small Results on short-text benchmark datasets. See the \suppl for full results. TT refers to training time in hours on a single Nvidia V100 GPU. $-$ denotes that results are unavailable.}
    \label{tab:results_repo_short_text}
      \centering
      \resizebox{\linewidth}{!}{
        \begin{tabular}{@{}l|ccccccccccc@{}}
        \toprule
        \textbf{Method} & \textbf{P@1} & \textbf{P@3} & \textbf{P@5} & \textbf{N@3} & \textbf{N@5} & \textbf{PSP@1} & \textbf{PSP@3} & \textbf{PSP@5} & \textbf{PSN@3} & \textbf{PSN@5} & \textbf{TT} \\ \midrule

        \multicolumn{12}{c}{LF-AmazonTitles-1.3M}\\ \midrule
        \alg & \textbf{56.75} & \textbf{49.19} & \textbf{44.09} & \textbf{53.84} & \textbf{52.41} & 29.18 & 33.01 & 35.36 & 32.07 & 33.91 &  97.75 \\
        DPR & 51.87 & 45.85 & 41.34 & 50.19 & 49.24 & 29.93 & 34.49 & 37.08 & 33.43 & 35.48  &  96.83 \\ 
        TAS & 51.2 & 44.65 & 40 & 48.88 & 47.62 & 28.53 & 32.03 & 34 & 31.17 & 32.76 & 96.87\\
        ANCE &  53.32 & 46.61 & 40.24 & 51.3 & 49.11 & \textbf{31.47} & \textbf{34.97} & \textbf{35.67} & \textbf{34.41} & \textbf{35.57}  & 447.25 \\ 
        SiameseXML & 49.02 & 42.72 & 38.52 & 46.38 & 45.15 & 27.12	& 30.43 & 32.52 & 29.41 & 30.9 & 9.89 \\ 
        ECLARE & 50.14 & 44.09 & 40 & 47.75 & 46.68 & 23.43 & 27.9 & 30.56 & 26.67 & 28.61 & 70.59 \\ 
        DECAF & 50.67 & 44.49 & 40.35 & 48.05 & 46.85 & 22.07 & 26.54 & 29.3 & 25.06 & 26.85 & 74.47 \\ 
        XR-Transformer & 50.14 & 44.07 & 39.98 & 47.71 & 46.59 & 20.06 & 24.85 & 27.79 & 23.44 & 25.41 & 132 \\ 
        LightXML & $-$ & $-$ & $-$ & $-$ & $-$ & $-$ & $-$ & $-$ & $-$ & $-$ & $-$ \\
        Astec & 48.82 & 42.62 & 38.44 & 46.11 & 44.8 & 21.47 & 25.41 & 27.86 & 24.08 & 25.66 & 18.54 \\ 
        Bonsai & 47.87 & 42.19 & 38.34 & 45.47 & 44.35 & 18.48 & 23.06 & 25.95 & 21.52 & 23.33 & 7.89 \\  
        \midrule
        
        \multicolumn{12}{c}{LF-AmazonTitles-131K}\\ \midrule
        \alg & \textbf{46.01} & \textbf{30.28} & \textbf{21.47} & \textbf{46.69} & \textbf{48.67} & \textbf{38.81} & \textbf{44.4} & \textbf{49.43} & \textbf{42.79} & \textbf{45.31} & 12.59 \\ 
        SiameseXML & 41.42 & 27.92 & 21.21 & 42.65 & 44.95 & 35.80 & 40.96 & 46.19 & 39.36 & 41.95 & 1.08 \\ 
        ECLARE & 40.74 & 27.54 & 19.88 & 42.01 & 44.16 & 33.51 & 39.55 & 44.7 & 37.7 & 40.21 & 2.16 \\ 
        DECAF & 38.4 & 25.84 & 18.65 & 39.43 & 41.46 & 30.85 & 36.44 & 41.42 & 34.69 & 37.13 & 2.16 \\ 
        XR-Transformer & 38.1 & 25.57 & 18.32 & 38.89 & 40.71 & 28.86 & 34.85 & 39.59 & 32.92 & 35.21 & 35.4 \\ 
        LightXML & 35.6 & 24.15 & 17.45 & 36.33 & 38.17 & 25.67 & 31.66 & 36.44 & 29.43 & 31.68 & 71.4 \\ 
        BERTXML & 38.89	& 26.17	& 18.72	& 39.93	& 41.79	& 30.1	& 36.81 &	41.85 &	34.8 & 37.28 & 12.55\\
        Astec & 37.12 & 25.2 & 18.24 & 38.17 & 40.16 & 29.22 & 34.64 & 39.49 & 32.73 & 35.03 & 1.83 \\ 
        Bonsai & 34.11 & 23.06 & 16.63 & 34.81 & 36.57 & 24.75 & 30.35 & 34.86 & 28.32 & 30.47 & 0.1 \\ 
        \bottomrule
    \end{tabular}
    }
\end{table}

\begin{table}[t]
    \centering
    \caption{Results on the PR-85M dataset for personalized ad recommendations}
	\label{tab:results_prop}
    % \begin{adjustbox}{max width=\textwidth}
        \begin{tabular}{@{}l|ccccccc}
        \toprule
        \textbf{Method} & \textbf{P@1} & \textbf{P@3} & \textbf{P@5} &  \textbf{N@5} & \textbf{R@5} & \textbf{PSP@1} & \textbf{PSP@5} \\
        \midrule
        \alg & \textbf{30.77} & \textbf{18.09} & \textbf{13.20} & \textbf{32.46} & \textbf{33.82} & 24.90 & 32.87 \\ 
        ANCE & 24.97 & 15.52 & 11.72 & 28.68 & 31.64 & \textbf{26.48} & \textbf{33.55}\\ 
        SiameseXML & 27.60 & 16.45 & 12.59 & 30.28 & 30.61 & 17.46 & 24.89 \\
        \bottomrule
        \end{tabular}
    % \end{adjustbox}
\end{table}

\textbf{Evaluation metrics}: Algorithms were evaluated using popular metrics such as precision@$k$ (P@$k$, $k \in {1, 5}$) and their propensity-scored variants precision@$k$ (PSP@$k$, $k \in {1, 5}$). Results on other metrics such as nDCG@$k$ (N@$k$) \& propensity scored nDCG@$k$ (PSN@$k$) are included in Section~\ref{app:eval} in the \suppl. Definitions of all these metrics are available at \cite{XMLRepo} and definitions of metrics used in A/B testing experiments are provided in Section~\ref{app:eval} in the \suppl. 

\textbf{Offline evaluation on benchmark XC datasets}: 
\alg's P@1 could be up to 16\% higher than leading Siamese methods for XC including SiameseXML, ECLARE, and DECAF which are the focus of the paper. This demonstrates that \alg's design choices lead to significant gains over these methods. Please refer to the ablation experiments in Section~\ref{app:results} for detailed discussion on impact of \alg's components.  

Table~\ref{tab:results_repo_full_text} presents results on full-text datasets where \alg could also be upto 13\% and 15\% more accurate in P@$k$ and PSP@$k$ respectively when compared to leading deepXC methods such as LightXML and XR-Transformer. It is notable that these methods also use transformer architectures indicating the effectiveness of \alg's training pipeline. It is also notable that \alg outperforms a range of other negative sampling algorithms such as TAS \cite{hofstatter21}, DPR \cite{Karpukhin20} and ANCE~\cite{Zhang21}.
Similar results were observed on short-text datasets, where \alg could be up to 11\% and 13\% more accurate in P@1 and PSP@1 respectively, compared to specialized algorithms designed for short-texts including SiameseXML, Astec, DECAF, {\it etc}. \alg also continues to outperform other negative sampling algorithms such as TAS \cite{hofstatter21}, DPR \cite{Karpukhin20}, and ANCE~\cite{Zhang21}.

Curiously, \cite{Dahiya21} observed that jointly training a high-capacity feature extractor such as \bert along with the classification layer (as opposed to training them in a modular manner as done by \alg) could yield inferior results, especially on short-text datasets. We observe a similar trend where \alg's pipeline yielded 7\% better P@1 as compared to the same architecture trained in an end-to-end manner (referred as BERTXML in Table~\ref{tab:results_repo_short_text}). The two-stage training employed by \alg where the transformer-based encoder is first trained in a Siamese fashion was found to address this challenge and yield state-of-the-art results on short-text datasets as well.

\textbf{Analysis of Results}: Tail labels contribute significantly to the performance of both components of \alg's classifier -- $\cE_{\vtheta}(\vz_l)$ and $\Delta \vw_l$. Fig.~\ref{fig:decile} in the appendix presents decile-wise analysis indicating that NGAME derives its superior performance not just by predicting popular labels but predicting rare labels accurately as well. The same was corroborated by NGAME's performance in live A/B testing on Sponsored Search where it was able to predict labels that were not being predicted by the existing ensemble of methods. \alg's final predictions which make use of both components get the best out of both leading to overall superior performance. 

\textbf{Comparison to other Negative Mining Algorithms}: Results from Tables~\ref{tab:results_repo_short_text}, \ref{tab:results_repo_full_text} and Figure~\ref{fig:convergence} establish that \alg offers superior accuracies as well as faster convergence compare to a range of existing negative mining algorithms such as TAS \cite{hofstatter21}, DPR \cite{Karpukhin20}, O-SGD \cite{Kawaguchi20}, and ANCE~\cite{Zhang21}.

\textbf{Live A/B testing and offline evaluation on Personalized Ad Recommendation}: \alg was used to predict queries that could lead to clicks on a given webpage in the pipeline to show personalized ads to users and was compared to an existing ensemble of state-of-the-art IR, dense retrieval (DR) and XC techniques. In A/B tests on live traffic, \alg was found to increase Click-Through Rate (CTR) and Click Yield (CY) by 23\% and 19\% respectively. In manual labeling by expert judges, \alg was found to increase the quality of predictions, measured in terms of fraction of excellent and good predictions, by 16\% over the ensemble of baselines. The PR-85M dataset was created to capture this inverted prediction task by mining the search engine logs for a specific time period where each webpage title became a data point and search engine queries that led to a click on that webpage became labels relevant to that data point. \alg was found to be at least 2\% more accurate than leading XC as well as Siamese encoder-based methods including ANCE and SiameseXML in R@5 metric on the PR-85M dataset. Please see Table~\ref{tab:results_prop} for detailed results.

\textbf{Live A/B testing on Sponsored Search}\label{par:live_ab_testing}: \alg was deployed on a popular search engine and A/B tests were performed on live search engine traffic for matching user queries to advertiser bid phrases (Query2Bid). \alg was compared to an ensemble of leading (proprietary) IR, XC, generative and graph-based techniques. \alg was found to increase Impression Yield (IY), CY and Query Coverage (QC) by 1.3\%, 1.23\% and 2.12\%, respectively. The IY boost indicates that \alg was able to discover more ads which were not being captured by the existing ensemble of algorithms. The CY boost indicates that ads surfaced by \alg were more relevant to the end user. The QC boost indicates that \alg impressed ads on several queries for which ads were previously not being shown.

\section{Conclusions and Future Work}
\label{sec:broad_impact}
This work accelerates the training of XC architectures that use large Siamese encoders such as transformers. A key step towards doing this is the identification of negative mining techniques as a bottleneck that forces mini-batch sizes to remain small and in turn, slowing down convergence. The paper proposes the \alg method that uses negative-mining-aware mini-batch creation to train Siamese XC methods and effectively train large encoder architectures such as transformers making effective use of label-text. There exist other forms of label-metadata that have been exploited in XC literature, for instance label hierarchies and correlation graphs. Although experiments show that \alg outperforms these methods by use of more powerful architectures alone, it remains to be seen how \alg variants using graph/tree metadata would perform. It would also be interesting to explore interpretability directions to better understand situations where \alg stands to improve. In terms of theoretical results, it is a tantalizing opportunity to relate the $(\epsilon,r)$-goodness of the embedding model to the training loss of the model. This would set up a virtuous cycle and possibly offer stronger convergence proofs since progress in terms of the training loss would improve the goodness of the model, in turn offering better negatives leading to faster training, and so on.

\bibliographystyle{unsrt}
\bibliography{references}

\clearpage
\appendix

\begin{table}
    \caption{Breakdown of computation costs of different negative mining methods}
    \label{tab:ablation_speedup}
      \centering
        \begin{tabular}{@{}l|ccc|ccc@{}}
        \toprule
        \textbf{Metric} & \textbf{DPR} & \textbf{\alg} &  \textbf{ANCE} & \textbf{DPR} & \textbf{\alg} &  \textbf{ANCE} \\
        \midrule
            & \multicolumn{3}{c|}{LF-Wikipedia-500K} & \multicolumn{3}{c}{LF-AmazonTitles-1.3M} \\  \midrule
            Epoch time & 4710s & 4710s & 6360s & 1092s & 1092s & 3120s \\
            Sampling overhead & - & 6s & 1131.22s & - & 12s & 262.16s \\ 
            Total time & 4710s & 4716s & 7491.22s & 1092s & 1104s & 3382.16s \\ 
            Fraction Increase & - & 1.00 & 1.59 & - & 1.01 & 3.10 \\
            \bottomrule
        \end{tabular}
\end{table}

\section{Results}
\label{app:results}

\begin{table}[t]
    \caption{\small Detailed results on short-text datasets. TT refers to training time in hours on a single Nvidia V100 GPU.}
    %A `-' symbol indicates that the method could not scale to the dataset within a week on a single A100 GPU.}
    \label{tab:supp:results_short_text}
      \centering
      \resizebox{\linewidth}{!}{
        \begin{tabular}{@{}l|ccccccccccc@{}}
        \toprule
        \textbf{Method} & \textbf{P@1} & \textbf{P@3} & \textbf{P@5} & \textbf{N@3} & \textbf{N@5} & \textbf{PSP@1} & \textbf{PSP@3} & \textbf{PSP@5} & \textbf{PSN@3} & \textbf{PSN@5} & \textbf{TT} \\ \midrule
        
        \midrule
        \multicolumn{12}{c}{LF-AmazonTitles-1.3M}\\ \midrule
        \alg & 56.75 & 49.19 & 44.09 & 53.84 & 52.41 & 29.18 & 33.01 & 35.36 & 32.07 & 33.91 &  97.75 \\
        SiameseXML & 49.02 & 42.72 & 38.52 & 46.38 & 45.15 & 27.12	& 30.43 & 32.52 & 29.41 & 30.9 & 9.89 \\ 
        ECLARE & 50.14 & 44.09 & 40 & 47.75 & 46.68 & 23.43 & 27.9 & 30.56 & 26.67 & 28.61 & 70.59 \\ 
        GalaXC & 49.81 & 44.23 & 40.12 & 47.64 & 46.47 & 25.22 & 29.12 & 31.44 & 27.81 & 29.36 & 9.55 \\ 
        DECAF & 50.67 & 44.49 & 40.35 & 48.05 & 46.85 & 22.07 & 26.54 & 29.3 & 25.06 & 26.85 & 74.47 \\ 
        Astec & 48.82 & 42.62 & 38.44 & 46.11 & 44.8 & 21.47 & 25.41 & 27.86 & 24.08 & 25.66 & 18.54 \\ 
        AttentionXML & 45.04 & 39.71 & 36.25 & 42.42 & 41.23 & 15.97 & 19.9 & 22.54 & 18.23 & 19.6 & 380.02 \\ 
        MACH & 35.68 & 31.22 & 28.35 & 33.42 & 32.27 & 9.32 & 11.65 & 13.26 & 10.79 & 11.65 & 60.39 \\ 
        X-Transformer & - & - & - & - & - & - & - & - & - & - & - \\ 
        LightXML & - & - & - & - & - & - & - & - & - & - & - \\ 
        AnneXML & 47.79 & 41.65 & 36.91 & 44.83 & 42.93 & 15.42 & 19.67 & 21.91 & 18.05 & 19.36 & 2.48 \\ 
        DiSMEC & - & - & - & - & - & - & - & - & - & - & - \\ 
        Parabel & 46.79 & 41.36 & 37.65 & 44.39 & 43.25 & 16.94 & 21.31 & 24.13 & 19.7 & 21.34 & 1.5 \\ 
        XT & 40.6 & 35.74 & 32.01 & 38.18 & 36.68 & 13.67 & 17.11 & 19.06 & 15.64 & 16.65 & 82.18 \\ 
        Slice & 34.8 & 30.58 & 27.71 & 32.72 & 31.69 & 13.96 & 17.08 & 19.14 & 15.83 & 16.97 & 0.79 \\ 
        PfastreXML & 37.08 & 33.77 & 31.43 & 36.61 & 36.61 & 28.71 & 30.98 & 32.51 & 29.92 & 30.73 & 9.66 \\ 
        Bonsai & 47.87 & 42.19 & 38.34 & 45.47 & 44.35 & 18.48 & 23.06 & 25.95 & 21.52 & 23.33 & 7.89 \\ 
        XR-Transformer & 50.14 & 44.07 & 39.98 & 47.71 & 46.59 & 20.06 & 24.85 & 27.79 & 23.44 & 25.41 & 132 \\ 
        DPR & 51.87 & 45.85 & 41.34 & 50.19 & 49.24 & 29.93 & 34.49 & 37.08 & 33.43 & 35.48  &  96.83 \\ 
        TAS & 51.2 & 44.65 & 40 & 48.88 & 47.62 & 28.53 & 32.03 & 34 & 31.17 & 32.76 & 96.87\\
        ANCE &  53.32 & 46.61 & 40.24 & 51.3 & 49.11 & 31.47 & 34.97 & 35.67 & 34.41 & 35.57  & 447.25 \\

        \midrule
        \multicolumn{12}{c}{LF-AmazonTitles-131K}\\ \midrule
        \alg & 46.01 & 30.28 & 21.47 & 46.69 & 48.67 & 38.81 & 44.4 & 49.43 & 42.79 & 45.31 & 12.59 \\ 
        SiameseXML & 41.42 & 27.92 & 21.21 & 42.65 & 44.95 & 35.80 & 40.96 & 46.19 & 39.36 & 41.95 & 1.08 \\ 
        ECLARE & 40.74 & 27.54 & 19.88 & 42.01 & 44.16 & 33.51 & 39.55 & 44.7 & 37.7 & 40.21 & 2.16 \\ 
        GalaXC & 39.17 & 26.85 & 19.49 & 40.82 & 43.06 & 32.5 & 38.79 & 43.95 & 36.86 & 39.37 & 0.42 \\ 
        DECAF & 38.4 & 25.84 & 18.65 & 39.43 & 41.46 & 30.85 & 36.44 & 41.42 & 34.69 & 37.13 & 2.16 \\ 
        Astec & 37.12 & 25.2 & 18.24 & 38.17 & 40.16 & 29.22 & 34.64 & 39.49 & 32.73 & 35.03 & 1.83 \\ 
        AttentionXML & 32.25 & 21.7 & 15.61 & 32.83 & 34.42 & 23.97 & 28.6 & 32.57 & 26.88 & 28.75 & 20.73 \\ 
        MACH & 33.49 & 22.71 & 16.45 & 34.36 & 36.16 & 24.97 & 30.23 & 34.72 & 28.41 & 30.54 & 3.3 \\ 
        X-Transformer & 29.95 & 18.73 & 13.07 & 28.75 & 29.6 & 21.72 & 24.42 & 27.09 & 23.18 & 24.39 & 64.4 \\ 
        LightXML & 35.6 & 24.15 & 17.45 & 36.33 & 38.17 & 25.67 & 31.66 & 36.44 & 29.43 & 31.68 & 71.4 \\ 
        BERTXML & 38.89	& 26.17	& 18.72	& 39.93	& 41.79	& 30.1	& 36.81 &	41.85 &	34.8 & 37.28 & 12.55\\
        AnneXML & 30.05 & 21.25 & 16.02 & 31.58 & 34.05 & 19.23 & 26.09 & 32.26 & 23.64 & 26.6 & 0.08 \\ 
        DiSMEC & 35.14 & 23.88 & 17.24 & 36.17 & 38.06 & 25.86 & 32.11 & 36.97 & 30.09 & 32.47 & 3.1 \\ 
        Parabel & 32.6 & 21.8 & 15.61 & 32.96 & 34.47 & 23.27 & 28.21 & 32.14 & 26.36 & 28.21 & 0.03 \\ 
        XT & 31.41 & 21.39 & 15.48 & 32.17 & 33.86 & 22.37 & 27.51 & 31.64 & 25.58 & 27.52 & 9.46 \\ 
        Slice & 30.43 & 20.5 & 14.84 & 31.07 & 32.76 & 23.08 & 27.74 & 31.89 & 26.11 & 28.13 & 0.08 \\ 
        PfastreXML & 32.56 & 22.25 & 16.05 & 33.62 & 35.26 & 26.81 & 30.61 & 34.24 & 29.02 & 30.67 & 0.26 \\ 
        Bonsai & 34.11 & 23.06 & 16.63 & 34.81 & 36.57 & 24.75 & 30.35 & 34.86 & 28.32 & 30.47 & 0.1 \\ 
        XR-Transformer & 38.1 & 25.57 & 18.32 & 38.89 & 40.71 & 28.86 & 34.85 & 39.59 & 32.92 & 35.21 & 35.4 \\
        RocketQA & 42.75 & - & 20.98 & - & 46.86 & 38.84 & - & 48.84 & - & - & - \\
        \bottomrule
    \end{tabular}}
\end{table}

\begin{table}[t]
    \caption{\small Detailed results on long-text datasets. TT refers to training time in hours on a single Nvidia V100 GPU.}
    % A `-' symbol indicates that the method could not scale to the dataset within a week on a single A100 GPU.}
    \label{tab:supp:results_full_text}
      \centering
      \resizebox{\linewidth}{!}{
        \begin{tabular}{@{}l|ccccccccccc@{}}
        \toprule
        \textbf{Method} & \textbf{P@1} & \textbf{P@3} & \textbf{P@5} & \textbf{N@3} & \textbf{N@5} & \textbf{PSP@1} & \textbf{PSP@3} & \textbf{PSP@5} & \textbf{PSN@3} & \textbf{PSN@5} & \textbf{TT} \\ \midrule
        
        \midrule
        \multicolumn{12}{c}{LF-Wikipedia-500K}\\ \midrule
        \alg & \textbf{84.01} & \textbf{64.69} & \textbf{49.97} & \textbf{78.25} & \textbf{75.97} & \textbf{41.25} & \textbf{52.57} & \textbf{57.04} & \textbf{51.58} & \textbf{56.11} & 54.88  \\ 
        SiameseXML & 67.26 & 44.82 & 33.73 & 56.64 & 54.29 & 33.95 & 35.46 & 37.07 & 36.58 & 38.93 & 4.37 \\ 
        ECLARE & 68.04 & 46.44 & 35.74 & 58.15 & 56.37 & 31.02 & 35.39 & 38.29 & 35.66 & 38.72 & 9.4 \\ 
        GalaXC & 55.26 & 35.07 & 26.13 & 45.51 & 43.7 & 31.82 & 31.26 & 32.47 & 32.75 & 34.5 & - \\ 
        DECAF & 73.96 & 54.17 & 42.43 & 66.31 & 64.81 & 32.13 & 40.13 & 44.59 & 39.57 & 43.7 & 44.23 \\ 
        Astec & 73.02 & 52.02 & 40.53 & 64.1 & 62.32 & 30.69 & 36.48 & 40.38 & 36.33 & 39.84 & 20.35 \\ 
        AttentionXML & 82.73 & 63.75 & 50.41 & 76.56 & 74.86 & 34 & 44.32 & 50.15 & 42.99 & 47.69 & 110.6 \\ 
        MACH & 52.78 & 32.39 & 23.75 & 42.05 & 39.7 & 17.65 & 18.06 & 18.66 & 19.18 & 20.45 & - \\ 
        X-Transformer & 76.95 & 58.42 & 46.14 & - & - & - & - & - & - & - &  \\ 
        LightXML & 81.59 & 61.78 & 47.64 & 74.73 & 72.23 & 31.99 & 42 & 46.53 & 40.99 & 45.18 & 185.56 \\ 
        AnneXML & 64.64 & 43.2 & 32.77 & 54.54 & 52.42 & 26.88 & 30.24 & 32.79 & 30.71 & 33.33 & 15.50 \\ 
        DiSMEC & 70.2 & 50.6 & 39.7 & 42.1 & 40.5 & 31.2 & 33.4 & 37 & 33.7 & 37.1 &  - \\ 
        Parabel & 68.7 & 49.57 & 38.64 & 60.51 & 58.62 & 26.88 & 31.96 & 35.26 & 31.73 & 34.61 & 2.72 \\ 
        XT & 64.48 & 45.84 & 35.46 & - & - & - & - & - & - & - & - \\ 
        PfastreXML & 59.5 & 40.2 & 30.7 & 30.1 & 28.7 & 29.2 & 27.6 & 27.7 & 28.7 & 28.3 & 63.59 \\ 
        Bonsai & 69.2 & 49.8 & 38.8 & 60.99 & 59.16 & 27.46 & 32.25 & 35.48 & - & - & -\\ 
        XR-Transformer & 81.62 & 61.38 & 47.85 & 74.46 & 72.43 & 33.58 & 42.97 & 47.81 & 42.21 & 46.61 & 318.9 \\ 
        DPR & 79.91 & 59.51 & 45.9 & 72.69 & 70.58 & 37.57 & 46.51 & 50.7 & 45.96 & 50.16 & 54.67 \\ 
        TAS & 82.23 & 62.7 & 48.36 & 75.9 & 73.5 & 38.43 & 48.38 & 52.83 & 47.59 & 51.9 & 54.68 \\ 
        ANCE & 76.9 & 57.64 & 45.1 & 70.61 & 69.39 & 37.75 & 44.65 & 48.85 & 45.08 & 49.65  & 75.08 \\         
        \midrule
        \multicolumn{12}{c}{LF-WikiSeeAlso-320K}\\ \midrule
        \alg & \textbf{47.65} & \textbf{31.56} & \textbf{23.68} & \textbf{47.5} & \textbf{48.99} & \textbf{33.83} & \textbf{37.79} & \textbf{41.03} & \textbf{38.36} & \textbf{41.01} & 75.39 \\ 
        SiamseXML & 42.16 & 28.14 & 21.39 & 41.79 & 43.36 & 29.02 &	32.68 & 36.03 & 32.64 & 35.17 & 2.33 \\ 
        ECLARE & 40.58 & 26.86 & 20.14 & 40.05 & 41.23 & 26.04 & 30.09 & 33.01 & 30.06 & 32.32 & 9.4 \\ 
        GalaXC & 38.96 & 25.84 & 19.58 & 37.76 & 38.92 & 25.78 & 29.37 & 32.53 & 28.71 & 30.87 & 1.1 \\ 
        DECAF & 41.36 & 28.04 & 21.38 & 41.55 & 43.32 & 25.72 & 30.93 & 34.89 & 30.69 & 33.69 & 13.4 \\ 
        Astec & 40.07 & 26.69 & 20.36 & 39.36 & 40.88 & 23.41 & 28.08 & 31.92 & 27.48 & 30.17 & 6.39 \\ 
        AttentionXML & 40.5 & 26.43 & 19.87 & 39.13 & 40.26 & 22.67 & 26.66 & 29.83 & 26.13 & 28.38 & 90.37 \\ 
        MACH & 27.18 & 17.38 & 12.89 & 26.09 & 26.8 & 13.11 & 15.28 & 16.93 & 15.17 & 16.48 & 50.22 \\ 
        X-Transformer & - & - & - & - & - & - & - & - & - & - & - \\ 
        LightXML & 34.5 & 22.31 & 16.83 & 33.21 & 34.24 & 17.85 & 21.26 & 24.16 & 20.81 & 22.8 & 249 \\ 
        BERTXML & 42.63 & 27.65 & 20.41 & 41.8 & 42.88 & 26.16 & 31.41 & 34.63 & 31.2 & 33.8 & 116.67 \\
        AnneXML & 30.79 & 20.88 & 16.47 & 30.02 & 31.64 & 13.48 & 17.92 & 22.21 & 16.52 & 19.08 & 2.4 \\ 
        DiSMEC & 34.59 & 23.58 & 18.26 & 34.43 & 36.11 & 18.95 & 23.92 & 27.9 & 23.04 & 25.76 & 58.79 \\ 
        Parabel & 33.46 & 22.03 & 16.61 & 32.4 & 33.34 & 17.1 & 20.73 & 23.53 & 20.02 & 21.88 & 0.33 \\ 
        XT & 30.1 & 19.6 & 14.92 & 28.65 & 29.58 & 14.43 & 17.13 & 19.69 & 16.37 & 17.97 & 3.27 \\ 
        Slice & 27.74 & 19.39 & 15.47 & 27.84 & 29.65 & 13.07 & 17.5 & 21.55 & 16.36 & 18.9 & 0.2 \\ 
        PfastreXML & 28.79 & 18.38 & 13.6 & 27.69 & 28.28 & 17.12 & 18.19 & 19.43 & 18.23 & 19.2 & 4.97 \\ 
        Bonsai & 34.86 & 23.21 & 17.66 & 34.09 & 35.32 & 18.19 & 22.35 & 25.66 & 21.62 & 23.84 & 1.39 \\ 
        XR-Transformer & 42.57 & 28.24 & 21.3 & 41.99 & 43.44 & 25.18 & 30.13 & 33.79 & 29.84 & 32.59 & 119.47 \\
        \midrule
        \multicolumn{12}{c}{LF-Amazon-131K}\\ \midrule
        \alg & \textbf{46.53} & \textbf{30.89} & \textbf{22.02} & \textbf{47.44} & 49.58 & \textbf{38.53} & \textbf{44.95} & \textbf{50.45} & \textbf{43.07} & \textbf{45.81} & 39.99\\
        SiameseXML & 44.81 & 30.19 & 21.94 & 46.15 & 48.76 & 37.56 &	43.69 & 49.75 & 41.91 & 44.97 & 1.18 \\ 
        ECLARE & 43.56 & 29.65 & 21.57 & 45.24 & 47.82 & 34.98 & 42.38 & 48.53 & 40.3 & 43.37 & 2.15 \\ 
        GalaXC & 41.46 & 28.04 & 20.25 & 43.08 & 45.32 & 35.1 & 41.18 & 46.38 & 39.55 & 42.13 & 0.45 \\ 
        DECAF & 42.94 & 28.79 & 21 & 44.25 & 46.84 & 34.52 & 41.14 & 47.33 & 39.35 & 42.48 & 1.8 \\ 
        Astec & 42.22 & 28.62 & 20.85 & 43.57 & 46.06 & 32.95 & 39.42 & 45.3 & 37.45 & 40.35 & 3.05 \\ 
        AttentionXML & 42.9 & 28.96 & 20.97 & 44.07 & 46.44 & 32.92 & 39.51 & 45.24 & 37.49 & 40.33 & 50.17 \\ 
        MACH & 34.52 & 23.39 & 17 & 35.53 & 37.51 & 25.27 & 30.71 & 35.42 & 29.02 & 31.33 & 13.91 \\ 
        X-Transformer & - & - & - & - & - & - & - & - & - & - & - \\ 
        LightXML & 41.49 & 28.32 & 20.75 & 42.7 & 45.23 & 30.27 & 37.71 & 44.1 & 35.2 & 38.28 & 56.03 \\ 
        BERTXML & 42.59	& 28.39 & 20.27	& 43.57 & 45.61 & 33.55 & 40.83 & 46.4 & 38.8 & 41.61 & 48.11 \\
        AnneXML & 35.73 & 25.46 & 19.41 & 37.81 & 41.08 & 23.56 & 31.97 & 39.95 & 29.07 & 33 & 0.68 \\ 
        DiSMEC & 41.68 & 28.32 & 20.58 & 43.22 & 45.69 & 31.61 & 38.96 & 45.07 & 36.97 & 40.05 & 7.12 \\ 
        Parabel & 39.57 & 26.64 & 19.26 & 40.48 & 42.61 & 28.99 & 35.36 & 40.69 & 33.36 & 35.97 & 0.1 \\ 
        XT & 34.31 & 23.27 & 16.99 & 35.18 & 37.26 & 24.35 & 29.81 & 34.7 & 27.95 & 30.34 & 1.38 \\ 
        Slice & 32.07 & 22.21 & 16.52 & 33.54 & 35.98 & 23.14 & 29.08 & 34.63 & 27.25 & 30.06 & 0.11 \\ 
        PfastreXML & 35.83 & 24.35 & 17.6 & 36.97 & 38.85 & 28.99 & 33.24 & 37.4 & 31.65 & 33.62 & 1.54 \\ 
        Bonsai & 40.23 & 27.29 & 19.87 & 41.46 & 43.84 & 29.6 & 36.52 & 42.39 & 34.43 & 37.34 & 0.4 \\ 
        XR-Transformer & 45.61 & 30.85 & 22.32 & 47.1 & 49.65 & 34.93 & 42.83 & 49.24 & 40.67 & 43.91 & 38.4 \\ 
        \bottomrule
    \end{tabular}}
\end{table}

Tables~\ref{tab:supp:results_short_text} and \ref{tab:supp:results_full_text} present detailed results on all baseline methods for short and full-text datasets respectively. Table~\ref{tab:ablation_speedup} considers multiple negative sampling algorithms including DPR \cite{Karpukhin20}, ANCE \cite{Zhang21} and \alg and shows the sampling overhead incurred by each method. \alg incurs barely 1\% sampling overhead and offers epoch times similar to DPR that incurs no sampling overhead. However, \alg offers much better accuracies than DPR on a range of datasets (see Tables~\ref{tab:results_repo_full_text} and \ref{tab:results_repo_short_text}). On the other hand, ANCE incurs between $59 - 210\%$ sampling overhead slowing down the overall training procedure.

Figure~\ref{fig:convergence} gives convergence plots offered by various negative sampling methods. \alg offers by far the fastest convergence of all methods. Fig~\ref{fig:decile} shows how various deciles contribute to \alg's performance. It is notable that a bulk of \alg's (high) P@5 performance is due to accurate retrieval of tail labels. Figure~\ref{fig:negatives} gives an illustrative example of how \alg recovered hard negatives for an actual data point.

\begin{figure}[t]
      \centering
      \includegraphics[width=0.7\linewidth]{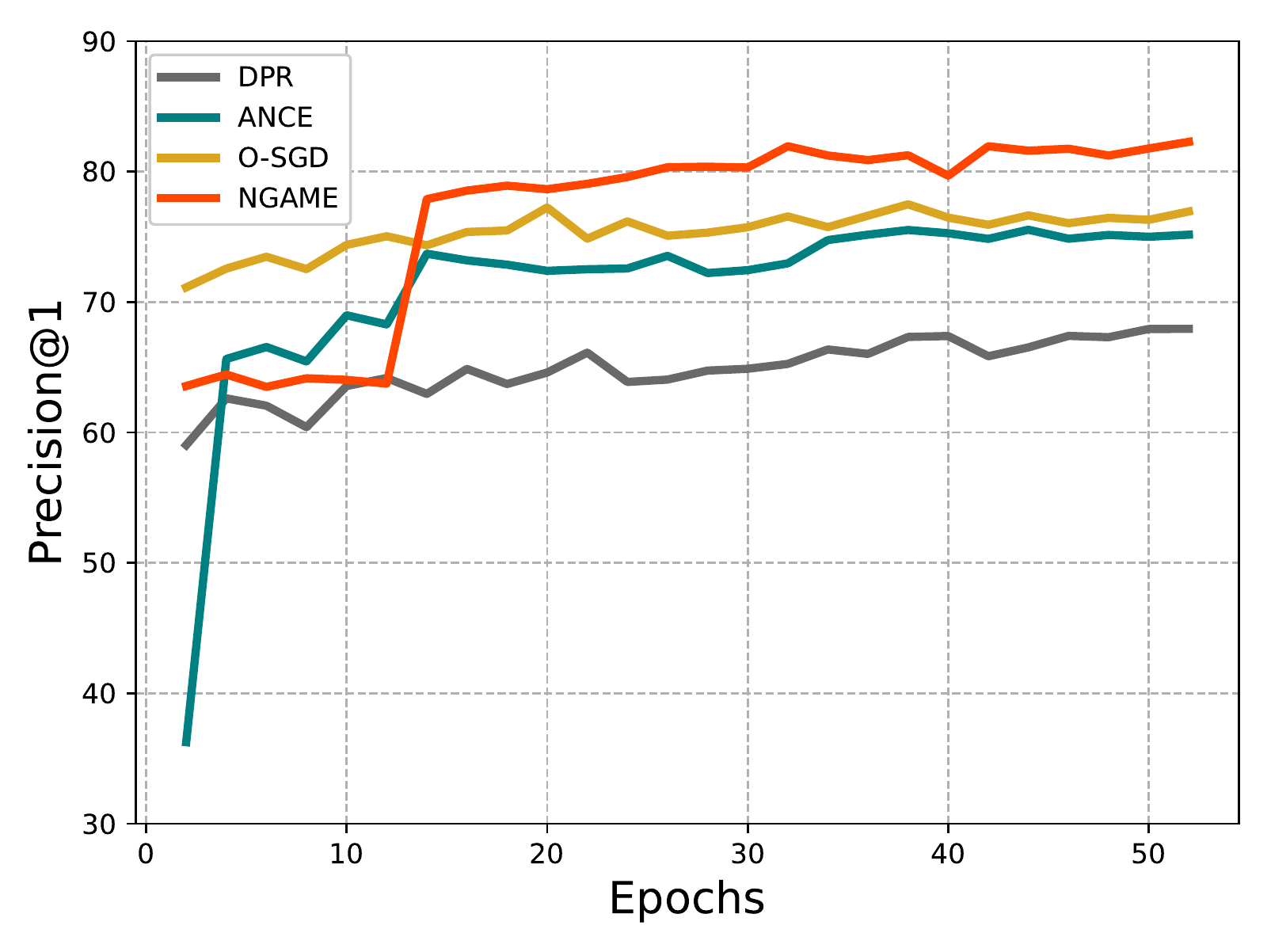}
        \caption{Convergence with different negative mining techniques on the LF-Wikipedia-500K dataset. Note that the plot include P@1 just for M1.}
        \label{fig:convergence}
\end{figure}

\begin{figure}[t]
    \begin{subfigure}{.43\textwidth}
      \centering
      \includegraphics[width=\linewidth]{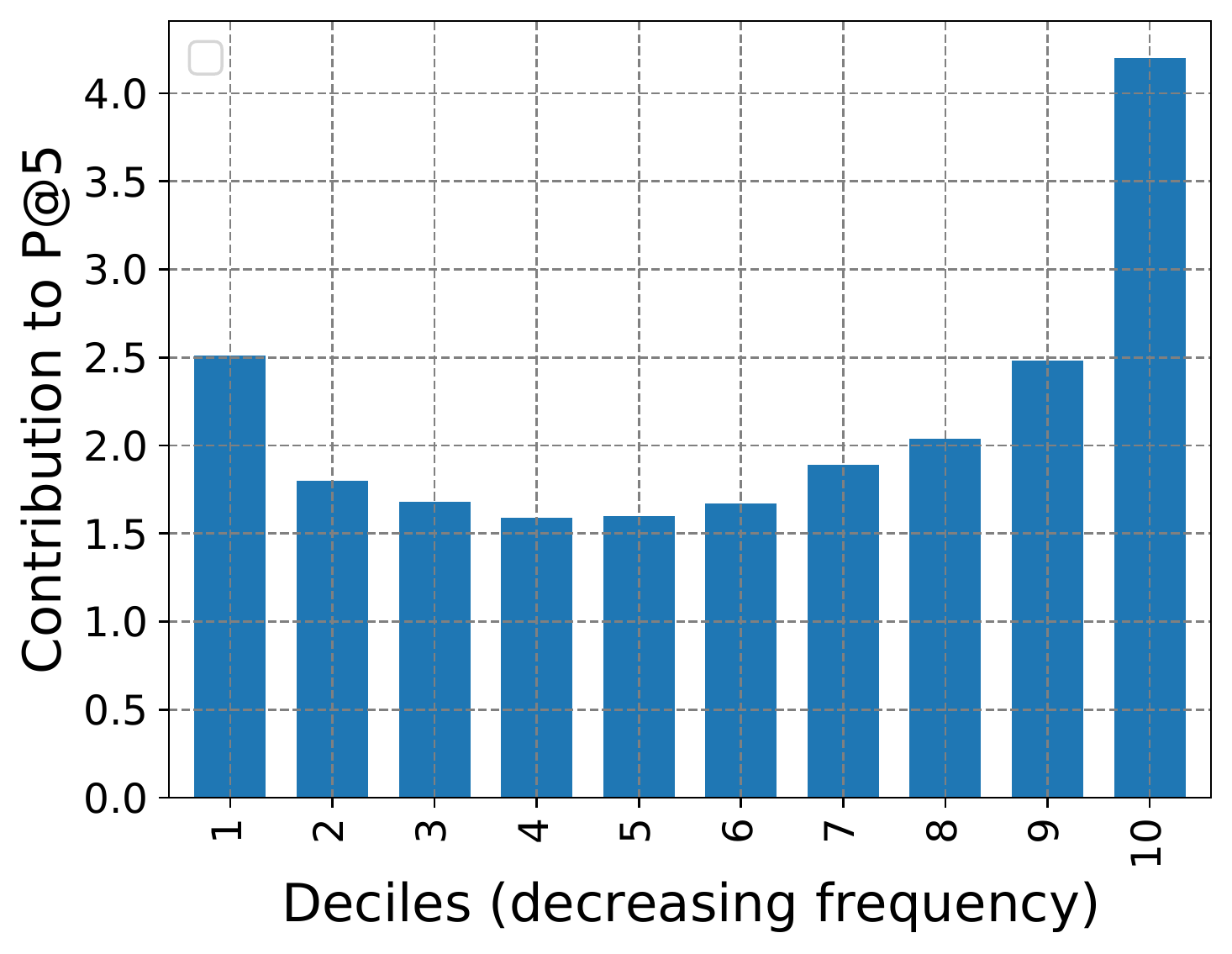}
      \caption{} \label{fig:decile}
    \end{subfigure}
    \begin{subfigure}{.5\textwidth}
      \centering
      \includegraphics[width=\linewidth]{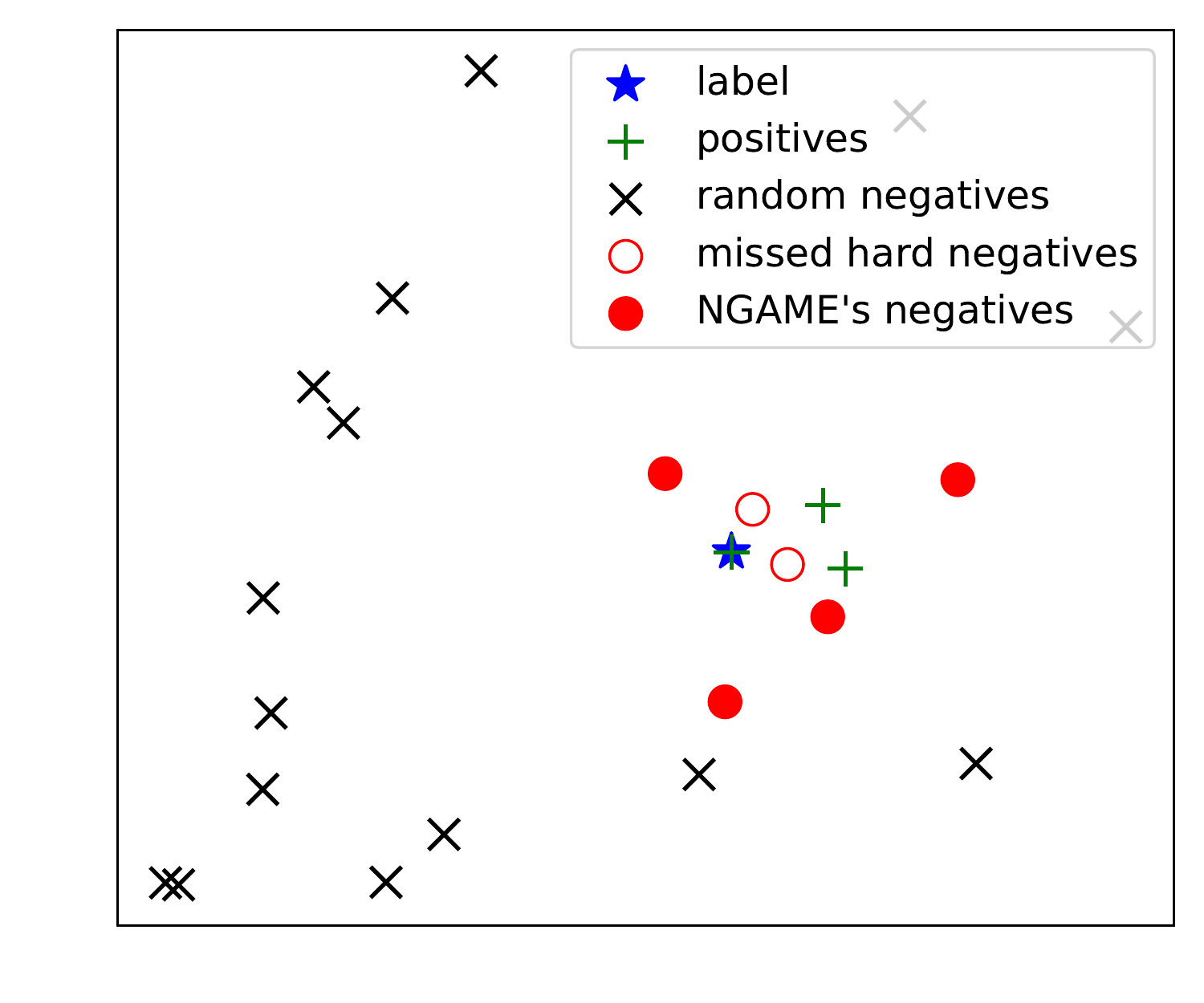}
      \caption{} \label{fig:negatives}
    \end{subfigure}
        \caption{(a) Performance of \alg on different deciles of LF-AmazonTitles-131K. It is notable that the tail deciles (that contain rare labels) contribute significantly to \alg's performance indicating that \alg derives its superior performance not just by predicting popular labels better but predicting rare labels accurately as well; (b) t-SNE~\cite{Van08} representation of positives, random-negatives, \alg's negatives and hard negatives missed by \alg for a product titled \emph{`Fearless Confessions: A Writer's Guide to Memoir'}. Note that \alg recovers most of the hard negatives for the data point.}
    \label{}
\end{figure}

\begin{table}[t]
    \caption{ Ablation study experimenting with prediction performance from different components of \alg. In all rows, training was done using \alg's usual pipeline. However, predictions were made using just label embeddings in the first row, using just the label classifier in the second row, and using score fusion function in the last row. \alg's score fusion strategy offers moderate gains in accuracy as compared to a simple sum of two scores (compare row 3 vs row 4).}
    \label{tab:ablation_fusion}
      \centering
    %   \resizebox{\linewidth}{!}{
        \begin{tabular}{@{}l|cc|cc@{}}
        \toprule
        % \textbf{Training} &
        \textbf{Prediction} & \textbf{P@1} & \textbf{P@5} & \textbf{P@1} & \textbf{P@5} \\
        \midrule
            & \multicolumn{2}{c|}{LF-AmazonTitles-131K} & \multicolumn{2}{c}{LF-AmazonTitles-1.3M}\\ \midrule
            $\cE_{\vtheta}(\vx)^{\top} \cE_{\vtheta}(\vz_l)$ & 42.61 & 20.69 & 45.82 & 35.48 \\
             $\cE_{\vtheta}(\vx)^{\top}\vw_l$ & 44.95 & 21.20 & 54.69 & 42.8 \\
            %  $\cE_{\vtheta}(\vx)^{\top}\cE_{\vtheta}(\vz_l) + \sigma(\cE_{\vtheta}(\vx)^{\top}\vw_l)$ & 21.51 & 48.62 & 57.83 & 39.42 & 47.34 & 33.37 \\
            %  $\cF_{PSP}(\cE_{\vtheta}(\vz_l)$,$\vw_l,\cE_{\vtheta}(\vx))$ & 44.41 & 21.61 & 41.24 & 50.38 & 49.71 & 42.06 & 41.04 & 44.89\\
             $\cF(\cE_{\vtheta}(\vz_l)$,$\vw_l,\cE_{\vtheta}(\vx))$ & \textbf{46.01} & \textbf{21.47} & \textbf{56.75} & \textbf{44.09} \\
            \bottomrule
        \end{tabular}
        % }
\end{table}

\section{Ablation Experiments}
\label{app:dummy section}
Table~\ref{tab:ablation_fusion} gives results of ablation experiments that establish that all of \alg's architectural components are essential to obtain optimal accuracy across several metrics and datasets. In particular, making predictions using label/data point embeddings alone is suboptimal. However, so is making predictions using classifiers scores alone. Moreover, combining classifier and embedding scores naively also does not yield optimal accuracy. A combination of embedding and classifier scores using \alg's score fusion strategy is what offers the best performance uniformly across several metrics and datasets.

\section{Implementation Details and Hyper-parameters}
\label{app:implementation}

Balanced hierarchical k-means was used in Algorithm~\ref{algo:iteration} and offered $N/C$ balanced clusters in $D$-dimensions in $\bigO{ND\log(N/C)}$ time.

\begin{table}
	\caption{ Hyper-parameter values for \alg on all datasets to enable reproducibility. \alg code will be released publicly. Most hyperparameters were set to their default values across all datasets. LR is learning rate. Multiple clusters were chosen to form a batch hence $B > C$ in general. Clusters were refreshed after $\tau$ epochs. Cluster size $C$ was doubled after every 25 epochs}
	\label{tab:hyperparams}
	\resizebox{\linewidth}{!}
	{
		\begin{tabular}{l|ccccccc}
			\toprule
			\textbf{Dataset} & 
			\makecell{\textbf{Batch}\\\textbf{Size} $S$} & 
			\makecell{\textbf{Cluster}\\ \textbf{Ref.} $\tau$}  & 
			\makecell{\textbf{Margin}\\$\gamma$} & 
			\makecell{\textbf{M1 epochs}\\ $epochs$} & 
			\makecell{\textbf{M1 LR}\\ $LR_1$} & 
			\makecell{\textbf{M2 LR}\\ $LR_2$} & 
			\makecell{\textbf{\bert seq.}\\\textbf{len} $L_{max}$} \\
% 			& \makecell{\textbf{M1 Epochs}\\$N_1$} &
% 			\makecell{\textbf{M2 Epochs}\\$N_2$}\\
			\midrule
			\multicolumn{7}{c}{Short-text benchmark datasets}\\ \midrule
            LF-AmazonTitles-131K & 1600 & 5 & 0.3 & 300 & 0.0001 & 0.001 & 32 \\ 
            LF-AmazonTitles-1.3M & -- do -- & -- do -- & -- do -- & 300 & -- do -- & -- do -- & -- do -- \\ 
            \midrule
			\multicolumn{7}{c}{Full-text benchmark datasets}\\ \midrule
            LF-Amazon-131K & 350 & -- do -- & -- do -- & 300 &-- do -- & -- do -- & 128 \\ 
            LF-WikiSeeAlso-320K & -- do -- & -- do -- & -- do -- & 300 & -- do -- & -- do -- & -- do -- \\ 
            LF-Wikipedia-500K & 256 & -- do -- & -- do -- & 40 &-- do -- & -- do -- & 256  \\
            \bottomrule
    	\end{tabular}
	}
\end{table}

\textbf{Hyper-parameters}: \alg's hyper-parameters include (cf. Algorithm~\ref{algo:iteration}): (1) cluster size $C$, (2) batch size $S$, (3) clustering refresh interval $\tau$ and (4) loss margin $\gamma$~\eqref{eq:loss}. Table~\ref{tab:hyperparams} presents the hyperparameter values for various datasets (most were set to default values). Both modules M1 and M2 used the Adam optimizer and hard negatives from \alg. Hyper-parameters for Adam include learning rate and number of epochs. \alg's encoder $\cE_{\vtheta}$ was initialized with a 6-layered DistilBERT base~\cite{Sanh2019DistilBERTAD, Wolf20} to encode the labels and data points in all experiments. The hyper-parameters of the other algorithms were set as suggested by their authors wherever applicable and by fine-grained validation otherwise.

\textbf{Implementing competing negative sampling algorithms}: Tables~\ref{tab:results_repo_full_text} and \ref{tab:results_repo_short_text} present results on several existing negative sampling algorithms such as DPR~\cite{Karpukhin20}, ANCE~\cite{Zhang21} and TAS~\cite{hofstatter21}. Each of these methods was afforded the same encoder and classifier architecture and loss function as \alg and the only difference in these implementations was that a different negative sampling technique was used in place of \alg.

\textbf{Module M1 (Encoder Training)}: This model simply trains the encoder model $\cE_{\vtheta}$. The residual vectors were fixed to zero in this module, effectively using the label embedding itself as the label classifier i.e. $\vw_l = \cE_{\vtheta}(\vz_l)$. For a data point $i \in [n]$, let $\cP^i_+$ denote its positive labels and $\cP^i_-$ denote the hard negatives offered by \alg. Algorithm~\ref{algo:iteration} was used to train the task loss as follows:\\
$\displaystyle
 \min_{\vtheta} \cL(\vtheta) = \sum_{i = 1}^N\sum_{l \in \cP^i_+}\sum_{k \in \cP^i_-} [\cE_{\vtheta}(\vx_i)^{\top}\cE_{\vtheta}(\vz_l) - \cE_{\vtheta}(\vx_i)^{\top}\cE_{\vtheta}(\vz_k) + \gamma]_+$\\
Note that only the hard negative labels of a data point offered by \alg and its positive labels are used for training. To accelerate training further, only one positive label was randomly sampled for each data point rather than using all positives, {\it i.e.,} $\abs{\cP^i_+} = 1$ was used -- note that this continued to offer an unbiased loss estimate.

\textbf{Module M2 (Classifier Training)}: Label classifiers were initialized to their embeddings, {\it i.e.,} $\vw_l \leftarrow \cE_{\hat\vtheta}(\vz_l)$ where $\hat\vtheta$ is the encoder model $\cE_{\vtheta}$ learnt by M1 (which is now frozen). Note that this is equivalent to initializing the residual vectors to zero. The original loss \eqref{eq:loss} was now minimized with respect to the residual vectors alone, effectively learning better classifiers than what the label embeddings alone could have provided. Say for a training data point $i \in [N]$, let $\cP^i_+$ denote its positive labels and $\cP^i_-$ its \alg hard negatives. The loss function used for module M2 training was:\\
$\displaystyle
\min_{\vW} \cL(\vW) = 
 \sum_{i = 1}^N\sum_{l \in \cP^i_+}\sum_{k \in \cP^i_-} [\vw_k^\top\cE_{\vtheta}(\vx_i) - \vw_l^\top\cE_{\vtheta}(\vx_i) + \gamma]_+$\\
In principle, one could use standard Binary Cross Entropy or Squared Hinge loss to train the classifiers. However, using the same objective in both modules yielded superior accuracies as compared to the alternatives. Additionally, the unit normalized classifiers offered by triplet loss were found to be more suitable for ANNS retrieval.

Note that since embeddings are frozen, M2 training can be distributed across multiple GPUs without any inter-GPU communication by splitting the label set into partitions and learning classifiers independently for each split.

\textbf{Inference (Score Fusion)}: Fusing embedding and classifier scores is commonly practiced~\cite{Dahiya21,Jain16}. However, unlike previous works that simply take a fixed linear combination of scores which is sub-optimal, \alg uses a fusion architecture $\cF$ that still offers $\bigO{b + D\log L}$ time inference. Once M1 and M2 training is over, an MIPS~\cite{MalkovY16} data structure is created over the label classifiers $\bc{\vw_l, l \in [L]}$. A small validation set $V$ of around 10K data points is created and for each data point $i \in V$, its embedding $\cE_{\vtheta}(\vx_i)$ is used to obtain a label shortlist $\hat\cN_i \subset [L]$ of $\bigO{\log L}$ labels from the MIPS structure (note that this shortlist may contain both relevant and irrelevant labels -- if encoder training has been accurate then this set will have high recall i.e. contain most relevant labels for this data point). Also, the set of labels relevant to this data point $\cP^i_+ = \bc{l \in [L]: y_{il} = +1}$ is obtained. For each data point-label pair $(i,l)$ with $i \in V, l \in \hat\cN_i \cup \cP^i_+$, a 3-dimensional descriptor $\vd_{il}$ is created as $\vd_{il} = (\cE_{\vtheta}(\vz_l)^\top\cE_{\vtheta}(\vx_i), \vw_l^\top\cE_{\vtheta}(\vx_i), f_l)$, where $f_l$ is the frequency of the label $l$ in the training set. The pair $(i,l)$ is assigned a tag $t_{il} = 1$ if the label $l \in \cP^i_+$ and $t_{il} = 0$ otherwise. A standard regression tree with depth 7 is learnt to predict $t_{il}$ using $\vd_{il}$. Let $\cT(\vd_{il})$ denote the score returned by this regressor. At test time, given a test data point $\vx$, its embedding $\cE_{\vtheta}(\vx)$ is computed and used to obtain a shortlist $\hat\cN$ of $\bigO{\log L}$ labels from the MIPS structure. To each shortlisted label $l \in \hat\cN$, a fused score is assigned as 
\begin{align*}
    \cF(\cE_{\vtheta}(\vx), \vw_l, \cE_{\vtheta}(\vz_l)) &= \cT(\cE_{\vtheta}(\vz_l)^\top\cE_{\vtheta}(\vx), \vw_l^\top\cE_{\vtheta}(\vx), f_l) \\
    &\qquad +\ \cE_{\vtheta}(\vz_l)^\top\cE_{\vtheta}(\vx) + \vw_l^\top\cE_{\vtheta}(\vx)
\end{align*}
Labels are then recommended in decreasing order of their fused scores. See Table~\ref{tab:ablation_fusion} for full results.

\section{Training and Inference Time Complexity for \alg}
\label{app:complexity}
\textbf{M1 Complexity}: We first calculate the epoch complexity of M1. In any mini-batch of say $S$ data points, computing embeddings of all $S$ data points and one positive label per data point takes $\bigO{bS}$ time if using an encoder model $\cE_{\vtheta}$ with $b$ parameters. Selecting hard negatives for all data points takes $\bigO{S^2}$ time. Executing backprop takes at most $\bigO{bS^2}$ time giving us an iteration complexity of $\bigO{bS^2}$. For an epoch of $N/S$ iterations, this takes $\bigO{NbS}$ time. \alg chooses $S = \bigO{\log L}$ giving us a $\bigO{Nb\log L}$ epoch time. We now compute the cost of refreshing the clustering from Algorithm~\ref{algo:iteration}. Computing $D$-dimensional embeddings for every data point takes $\bigO{ND + Nb}$ time. Clustering them takes $\bigO{ND\log N}$ time as mentioned above in Appendix~\ref{app:implementation}. Thus, the total time is $\bigO{Nb + ND\log N} = \bigO{Nb\log L}$ since $D \ll b$ and $L = \Theta(N)$. This establishes a $\bigO{Nb\log L}$ time complexity for M1 training.

\textbf{M2 Complexity}: The time complexity calculation is identical here except that executing backprop takes at most $\bigO{DS^2}$ time since only the classifiers (more specifically the residual vectors) are being optimized in M2. However, since $D \ll b$, we get a $\bigO{Nb\log L}$ time complexity for M2 training as well.

\textbf{Inference}: Given a test data point, obtaining its embedding takes $\bigO{b + D} = \bigO{b}$ time since $D \ll b$. Retrieving the $\log L$ top-ranked labels for this embedding from an ANNS/MIPS data structure takes $\bigO{D\log L}$ time \cite{MalkovY16}. Obtaining the Siamese and classifier scores for these shortlisted labels takes an additional $\bigO{D\log L}$ time bringing the total inference complexity to $\bigO{b + D\log L}$.

\section{Metrics for Live A/B Testing}
\label{app:eval}

Standard online metrics were used for A/B testing experiments. Impression Yield (IY) is defined as the rate at which an ad appears compared to the total number of searches. Similarly, Click Yield (CY) is defined as the rate at which an ad is clicked with respect to the total number of searches. Query Coverage (QC) is defined as the fraction of search queries for which at least one ad was shown. Formally, if we denote total number of searches as $SRPV$ :
    $$IY = \frac{\textit{Total No. of Ads}}{SRPV}$$
    $$CY = \frac{\textit{Total No. of Clicks}}{SRPV}$$
    $$QC = \frac{\textit{Total No. of Search Queries with atleast 1 Ad Shown}}{SRPV}$$

\clearpage

\clearpage

\section{Theoretical Analysis and Full Proofs}
\label{supp:theory}

\allowdisplaybreaks

This section details the theoretical analysis of the \alg method and also offers first-order stationarity guarantees under some standard, simplifying assumptions.

\subsection{Proof of Theorem~\ref{thm:main}: \alg offers Provably Accurate Hard-negatives}

Intuitively, an embedding is good if for most data points, its positive labels are embedded close to it and, a clustering is good if it does not split too many closely placed data point embedding vectors into distinct clusters.

To state and prove Theorem~\ref{thm:main}, we introduce the following notions of \emph{bad} events $E^r_{il}, F^r_{lm}, G^r_{il}$.
\begin{enumerate}
	\item For any data point $i \in [N]$, label $l \in [L]$, let $E^r_{il} := \bc{y_{il} = -1} \wedge \bc{\norm{\cE_{\vtheta}(\vx_i) - \cE_{\vtheta}(\vz_l)}_2 \leq r} \wedge \bc{l \notin \hat\cP^i_-}$ be the \textit{bad} event where label $l$ is an $r$-hard negative for data point $i$ but \alg fails to retrieve it in its shortlist $\hat\cP^i_-$.
	\item For any two data points $i, j \in [N]$, let $F^r_{ij} := \bc{\norm{\cE_{\vtheta}(\vx_i) - \cE_{\vtheta}(\vx_j)}_2 \leq r} \wedge \bc{c(i) \neq c(j)}$ be the bad even where the embeddings for data points $i,j$ are $r$-close to each other but clustering separates them into distinct clusters.
	\item For any data point $i \in [N]$ and any of its positive/relevant labels $l \in [L]$ i.e. where $y_{il} = +1$, let $G^r_{il} := \bc{\norm{\cE_{\vtheta}(\vz_l) - \cE_{\vtheta}(\vx_i)}_2 \geq r}$ be the bad event where label $l$ is relevant to data point $i$ but the encoder embeddings for the pair are more than $r$ distance far apart.
\end{enumerate}

We introduce some handy shortcuts to state the result.
\begin{itemize}
    \item For any data point $i \in [n]$, let $p_i \deff \abs{\bc{l \in [L]: y_{il} = +1}}$ be the number of relevant/positive labels for that data point.
    \item For any label $l \in [L]$, let $q_l \deff \abs{\bc{i \in [N]: y_{il} = +1}}$ be the number of data points for which this label is relevant/positive.
    \item Let $\bar p \deff \E{p_i}$ and $\bar q \deff \E{q_l}$ denote average/expected values of the quantities $p_i, q_l$ respectively.
    \item Let $p_{\min} = \min_{i \in [N]}\ p_i$ denote the smallest number of relevant labels for any data point.
    \item Let $q_{\min} = \min_{l \in [L]}\ q_l$ denote the smallest number of relevant data points for any label.

    \item Let $\mu_1 \deff \E{\frac{N-q_l}{q_l}}, \sigma_1^2 \deff \Var{\frac{N-q_l}{q_l}}, \sigma_2^2 \deff \Var{p_i}$ denote various handy second order moments involving the quantities $p_i, q_l$.
\end{itemize}

Given this we are ready to prove Theorem~\ref{thm:main} that establishes the negative mining guarantee for \alg. However, first we provide a full version of the theorem statement with all the details.

\begin{theorem}[Theorem~\ref{thm:main} restated with full constants.]
\label{thm:main-full}
Suppose Algorithm~\ref{algo:iteration} performs its clustering step using embeddings obtained using the encoder model parameterized by $\vtheta$ and identifies a set of hard negative labels $\hat\cP^i_-$ for data point $i$ in step 8 of the algorithm using the threshold $r > 0$. As defined above, for any label $l \in [L]$ and data point $i \in [N]$, let $E^r_{il} := \bc{y_{il} = -1} \wedge \bc{\norm{\cE_{\vtheta}(\vx_i) - \cE_{\vtheta}(\vz_l)}_2 \leq r} \wedge \bc{l \notin \hat\cQ^l_-}$ be the \textit{bad} event where label $l$ is an $r$-hard negative for data point $i$ but \alg fails to retrieve it in its shortlist $\hat\cP^i_-$. Then if the model $\vtheta$ was $(r,\epsilon_1)$-good and the clustering was $(2r,\epsilon_2)$-good, we are assured that
\[
\frac1{NL}\sum_{i \in [N], l \in [L]}\ind{E^r_{il}} \leq c_1\cdot\epsilon_1 + c_2\cdot\epsilon_2,
\]
where $c_1 = \frac{\bar q(\mu_1 + \sigma_1\sqrt L)}N$ and $c_2 = \frac{(\bar p + \sigma_2\sqrt N)N}{q_{\min}L}$.
\end{theorem}
As training proceeds, the quantity $\epsilon_1$ is expected to get smaller and smaller as the encoder model gets fine-tuned to the task. Although the above result is presented with reference to module M1, a similar result can be shown to hold true for module M2 as well with just one change -- the label representation used is $\vw_l = \cE_{\vtheta}(\vz_l) + \veta_l$ instead of $\cE_{\vtheta}(\vz_l)$. Also, note that the constants $c_1, c_2$ in the statement of Theorem~\ref{thm:main-full} depend purely on dataset characteristics and are entirely independent of the execution of the algorithm. The following corollary helps appreciate the result better by showing that these constants are small and in fact, upper bounded by unity for a simplified case.
\begin{corollary}[Corollary~\ref{cor:label-informal} restated with full constants.]
\label{cor:label}
For the special case when all data points have the same number of relevant labels i.e. $p_i = p$ for all all $i \in [N]$ as well as all labels are relevant to the same number of data points i.e. $q_l = q$ for all $l \in [L]$, we have $c_1, c_2 \leq 1$ which gives us
\[
\frac1{NL}\sum_{i \in [N], l \in [L]}\ind{E^r_{il}} \leq \epsilon_1 + \epsilon_2
\]
\end{corollary}
The proof of Theorem~\ref{thm:main-full} is given below followed by the proof of Corollary~\ref{cor:label}.

\begin{proof}[Proof (of Theorem~\ref{thm:main-full})]
Consider any label $l \in [L]$ and a data point $i \in [N]$ such that $y_{il} = -1$ i.e. label $l$ is not relevant to the data point $i$. Furthermore, let $j \in [N]$ be any other data point for which label $l$ is indeed relevant i.e. $y_{jl} = 1$. Then triangle inequality dictates that
\[
\norm{\cE_{\vtheta}(\vx_i) - \cE_{\vtheta}(\vz_l)}_2 \geq \norm{\cE_{\vtheta}(\vx_i) - \cE_{\vtheta}(\vx_j)}_2 - \norm{\cE_{\vtheta}(\vx_j) - \cE_{\vtheta}(\vz_l)}_2
\]
Now note that if $\norm{\cE_{\vtheta}(\vx_i) - \cE_{\vtheta}(\vx_j)}_2 > 2r$ (i.e. the two data point embeddings are far, which means the event $F^{2r}_{ij}$ does not occur) as well as $\norm{\cE_{\vtheta}(\vx_j) - \cE_{\vtheta}(\vz_l)}_2 < r$ (i.e. the label $l$ is embedded close to data point $j$, which means the event $G^r_{jl}$ does not occur), then the above inequality tells us that $\norm{\cE_{\vtheta}(\vx_i) - \cE_{\vtheta}(\vz_l)}_2 > r$ (i.e. the label $l$ is not a hard negative for data point $i$ rather an \emph{easy} one, which means the event $E^r_{il}$ cannot happen). Taking the contrapositive lets us conclude that
\[
E^r_{il} \Rightarrow F^{2r}_{ij} \vee G^r_{jl},
\]
or in other words,
\[
\ind{E^r_{il}} \leq \ind{F^{2r}_{ij}} + \ind{G^r_{jl}}
\]
Now note that the above must hold true for all data points $j$ such that $y_{jl} = 1$. In particular, if the event $E^r_{il}$ does occur i.e. Algorithm~\ref{algo:iteration} does miss the hard negative label $l$ for data point $i$, then by design of Algorithm~\ref{algo:iteration}, this can only happen if all data points for which label $l$ is relevant lie in clusters other than the cluster of data point $i$ itself since otherwise the label $l$ would have been discovered in step 8 of the algorithm. An averaging argument then lets us conclude that
\[
\ind{E^r_{il}} \leq \frac1{q_l}\sum_{j: y_{jl}=1} (\ind{F^{2r}_{ij}} + \ind{G^r_{jl}}),
\]
where we recall that $q_l$ is the number of data points for which label $l$ is relevant i.e. $q_l = \abs{\bc{j: y_{jl} = 1}}$. Summing over all data points $i \in [N]$ for which label $l$ could have been a missed hard negative gives us
\[
\sum_{i \in [N]}\ind{E^r_{il}} = \sum_{i: y_{il} \neq 1}\ind{E^r_{il}} \leq \frac1{q_l}\sum_{i: y_{il} \neq 1}\sum_{j: y_{jl}=1} \ind{F^{2r}_{ij}} + \frac{N-q_l}{q_l}\sum_{j: y_{jl}=1}\ind{G^r_{jl}}
\]
Further, summing over all data points $i$ and applying normalization gives us
\[
\frac1{NL}\sum_{i \in [N], l \in [L]}\ind{E^r_{il}} \leq \underbrace{\frac1{NL}\sum_{l \in [L]}\frac1{q_l}\sum_{i: y_{il} \neq 1}\sum_{j: y_{jl}=1} \ind{F^{2r}_{ij}}}_{(A)} + \underbrace{\frac1{NL}\sum_{l \in [L]}\frac{N-q_l}{q_l}\sum_{j: y_{jl}=1}\ind{G^r_{jl}}}_{(B)}
\]
Below we bound these two terms separately. To bound $(B)$, we recall the terms $\mu_1, \sigma_1^2, \bar p, \bar q$ defined earlier. Then we have
\begin{align*}
	(B) &= \mu_1\cdot\frac1{NL}\sum_{l \in [L]}\sum_{j: y_{jl}=1}\ind{G^r_{jl}} + \frac1{NL}\sum_{l \in [L]}\br{\frac{N-q_l}{q_l} - \mu_1}\sum_{j: y_{jl}=1}\ind{G^r_{jl}}\\
	&= \frac{\bar q\mu_1}{N}\cdot\epsilon_1 + \frac1{NL}\sum_{l \in [L]}\br{\frac{N-q_l}{q_l} - \mu_1}\sum_{j: y_{jl}=1}\ind{G^r_{jl}}\\
	&\leq \frac{\bar q\mu_1}{N}\cdot\epsilon_1 + \frac1{NL}\sqrt{\sum_{l \in [L]}\br{\frac{N-q_l}{q_l} - \mu_1}^2}\sqrt{\sum_{l \in [L]}\br{\sum_{j: y_{jl}=1}\ind{G^r_{jl}}}^2}\\
	&\leq \frac{\bar q\mu_1}N\cdot\epsilon_1 + \frac{\sigma_1\sqrt L}{NL}\sum_{l \in [L]}\sum_{j: y_{jl}=1}\ind{G^r_{jl}}\\
	&= \frac{\bar q(\mu_1 + \sigma_1\sqrt L)}N\cdot\epsilon_1,
\end{align*}
where in the second step, we used the fact that $\epsilon_1 = \frac1{\bar qL}\sum_{l \in [L]}\sum_{j: y_{jl}=1}\ind{G^r_{jl}}$ and $\bar qL = \bar pN$, in the third step we use the Cauchy-Schwartz inequality and in the fourth step we use the definition of $\sigma_1$, the fact that for any vector, its $L_2$ norm is always upper bounded by its $L_1$ norm, as well as the fact that the terms $G^r_{jl}$ are non-negative so that $\sum_{j: y_{jl}=1}\ind{G^r_{jl}} \geq 0$ for any $l \in [L]$.

To bound $(A)$, we recall the terms $q_{\min}, \bar p, \sigma_2^2$ defined earlier. Note that $q_{\min} \geq 1$. We get
\[
	(A) \leq \frac1{q_{\min}NL}\sum_{l \in [L]}\sum_{i: y_{il} \neq 1}\sum_{j: y_{jl}=1} \ind{F^{2r}_{ij}} = \frac1{q_{\min}NL}\sum_{i,j\in[N]}\sum_{\substack{l: y_{jl} = 1\\y_{il} \neq 1}} \ind{F^{2r}_{ij}}
\]
Now $\abs{\bc{l: y_{jl} = 1, y_{il} \neq 1}} \leq \abs{\bc{i: y_{jl} = 1}} = p_j$ which gives us
\begin{align*}
	(A) &\leq \frac1{q_{\min}NL}\sum_{i,j\in[N]} p_j\cdot\ind{F^{2r}_{ij}}\\
	&= \frac{\bar p}{q_{\min}NL}\sum_{i,j\in[N]}\ind{F^{2r}_{ij}} + \frac1{q_{\min}NL}\sum_{i,j\in[N]}(p_j - \bar p)\cdot\ind{F^{2r}_{ij}}\\
	&= \frac{\bar pN}{q_{\min}L}\cdot\epsilon_2 + \frac1{q_{\min}NL}\sum_{i,j\in[N]}(p_j - \bar p)\cdot\ind{F^{2r}_{ij}}\\
	&\leq \frac{\bar pN}{q_{\min}L}\cdot\epsilon_2 + \frac1{q_{\min}NL}\sqrt{\sum_{j \in [N]}(p_j - \bar p)^2}\sqrt{\sum_{j \in [N]}\br{\sum_{i \in [N]}\ind{F^{2r}_{ij}}}^2}\\
	&\leq \frac{\bar pN}{q_{\min}L}\cdot\epsilon_2 + \frac{\sigma_2}{q_{\min}L\sqrt N}\sqrt{\sum_{j \in [N]}\br{\sum_{i \in [N]}\ind{F^{2r}_{ij}}}^2}
\end{align*}
where in the third step we use the fact that $\epsilon_2 = \frac1{N^2}\sum_{i,j\in[N]}\ind{F^{2r}_{ij}}$, in the fourth step we use the Cauchy-Schwartz inequality, in the fifth step we use the definition of $\sigma_2$. Now, if $i,j,k \in [N]$ are independently selected uniformly from $[N]$, we have
\[
N^3\cdot\Ee{i,j,k}{\ind{F^{2r}_{ij}}\ind{F^{2r}_{ik}}} \leq N^4\cdot\br{\Ee{i,j}{\ind{F^{2r}_{ij}}}}\br{\Ee{i,k}{\ind{F^{2r}_{ik}}}} \leq N^4\epsilon_2^2.
\]
This gives us
\[
\sum_{j \in [N]}\br{\sum_{i \in [N]}\ind{F^{2r}_{ij}}}^2 \leq \sum_{i,j,k \in [N]} F^{2r}_{ij}F^{2r}_{ik} \leq N^4\epsilon_2^2
\]
This gives us
\[
(A) \leq \frac{(\bar p + \sigma_2\sqrt N)N}{q_{\min}L}\cdot\epsilon_2
\]
This finishes the proof. We note that the two of the steps taken above, that of upper bounding $q_l \geq q_{\min}$ and upper bounding $\abs{\bc{l: y_{jl} = 1, y_{il} \neq 1}} \leq p_j$ may be suboptimal and the proof may be tightened at these points. However, as Corollary~\ref{cor:label} shows, tightening these steps is not expected to change the essential result and is expected to only yield better values for the constants $c_1, c_2$.
\end{proof}

\begin{proof}[Proof (of Corollary~\ref{cor:label})]
We continue to use the terms $(A), (B)$ from the proof of Theorem~\ref{thm:main-full}. When $q_l = q$ for all $l \in [L]$, then we have $\mu_1 = \frac{N-q}q, \sigma_1 = 0, \bar q = q$ that gives us
\[
(B) \leq \frac qN\cdot\frac{N-q}q\cdot\epsilon_1 \leq \epsilon_1
\]
Similarly, when $p_i = p$ for all $i \in [N]$ and $q_l = q$ for all $l \in [L] $, we have $\sigma_2 = 0, \bar p = p$, and $q_{\min} = q = \bar q$, which gives us
\[
(A) \leq \frac{pN}{qL}\cdot\epsilon_2 = \epsilon_2,
\]
since $qL = \bar q L = \bar p N = pN$. This finishes the proof.
\end{proof}

\subsection{Proof of Theorem~\ref{thm:conv-informal}: \alg converges to a First-order Stationary Point}
In this section, we establish that under certain simplifying assumptions, \alg provably converges to a first-order stationary point.

We first list all the assumptions followed by a restatement of Theorem~\ref{thm:conv-informal} with all constants in Theorem~\ref{thm:conv}. To present the essential aspects of the proof, we will give the proof for a full-batch version of the algorithm, mini-batch extensions being straightforward:
\begin{assumption}[Full-batch Update with Eager Clustering]
\label{ass:eager-full}
Let \alg be executed so that it performs descent on the entire training set at once in a single descent step (and not in multiple steps using mini-batch SGD). Also, let the clustering in Algorithm~\ref{algo:iteration} be updated after each descent step i.e. $\tau = 1$
\end{assumption}
Our results require the training objective to be smooth. We recall the definition of an $H$-smooth function.
\begin{definition}[$H$-smooth objective]
A differentiable real-valued function $f: \cX \rightarrow \bR$ over any Euclidean space $\cX \subseteq \bR^p$ is said to be $H$-smooth if for all $\vx,\vy \in \cX$ we have
\[
f(\vx) \leq f(\vy) + \ip{\nabla f(\vy)}{\vx - \vy} + \frac H2\norm{\vx - \vy}_2^2
\]
\end{definition}
\begin{assumption}[Smooth Objective with Bounded Gradients]
\label{ass:smooth}
Let \alg be executed with training functions $\ell_{il}$ that are $H$-smooth for some finite $H > 0$. Also assume that the training functions $\ell_{il}$ have bounded gradients i.e. for some $G > 0$, we have $\norm{\nabla\ell_{il}(\vtheta)}_2 \leq G$ for all $i \in [N], l \in [L], \vtheta$.
\end{assumption}
We note that Assumption~\ref{ass:eager-full} is made to simplify the proof and not essential to the result whereas Assumption~\ref{ass:smooth} is standard in literature. We now present the formal statement.

\begin{theorem}[Restatement of Theorem~\ref{thm:conv-informal} with constants]
\label{thm:conv}
Suppose Assumptions~\ref{ass:eager-full} and \ref{ass:smooth} are satisfied and let $\epsilon_{\text{tot}}^t \deff c_1\epsilon_1^t + c_2\epsilon_2^t$ denote the total error assured by Theorem~\ref{thm:main-full} in terms of hard-negative terms missed by \alg at iteration $t$. Then there exists a smoothed objective $\tilde\cL$ (defined below) such that for some $\phi \in (0,1)$ we are assured that if \alg uses a sufficiently small step size $\eta < \frac{1-\phi}{2H(1+\phi^2)}$ where $H$ is the smoothness parameter from Assumption~\ref{ass:smooth}, then for any $T > 0$, within $T$ iterations, one of the following is assured
\begin{enumerate}
    \item For some iterate $t \leq T$, the \alg iterate $\vtheta^t$ assures
    \[
    \norm{\nabla\tilde\cL(\vtheta^t)}_2 \leq \frac{2(\tilde\cL(\vtheta^0)-\tilde\cL(\vtheta^\ast))}{\eta(1-\phi)T}
    \]
    \item For some iterate $t \leq T$, the \alg iterate $\vtheta^t$ assures
    \[
    \norm{\nabla\tilde\cL(\vtheta^t)}_2 \leq \frac{2G\epsilon_{\text{tot}}^t}\phi,
    \]
\end{enumerate}
where $\vtheta^\ast$ is the optimal model parameter, $\vtheta^0$ is the initial iterate and $G$ is the gradient bound in Assumption~\ref{ass:smooth}.
\end{theorem}
We note that $\epsilon_{\text{tot}}^t$ can be controlled directly by performing more relaxed clustering with larger clusters so that the clustering error $\epsilon_2^t \rightarrow 0$. In particular, note that if a trivial clustering is done where all points lie in the same single cluster, then indeed $\epsilon_2^t = 0$. The other component $\epsilon_1^t$ is controlled by the ability of the neural architecture to successfully embed data points and their positive labels in close vicinity which is also expected to continue improving as training progresses.

We will follow the following proof strategy:
\begin{enumerate}
	\item \textbf{Step 1}: Given a smooth training objective $\cL$, construct a smooth surrogate $\tilde\cL$ that places negligible weight on non-hard negatives.
	\item \textbf{Step 2}: Show that the training updates offered by \alg's hard-negative mining strategy implicitly perform gradient descent on $\tilde\cL$ but using biased gradient updates. Then show that under suitable assumptions, this bias is bounded.
	\item \textbf{Step 3}: Show that this results in \alg converging to an approximate first-order stationary point with respect to the modified objective $\tilde\cL$.
\end{enumerate}
The construction of such \emph{surrogate} objectives with respect to which convergence results are established, is popular in literature for instance in \cite{Kawaguchi20}. 
It is possible to relax the above to show convergence guarantees for the mini-batch SGD version of the algorithm, as well as when the clustering is updated after every few epochs (as is done in practice) instead of after every epoch, and would require additional bookkeeping to keep track of accumulated gradient bias due to \textit{stale} clustering and additional variance due to randomness mini-batch creation. We defer this more detailed analysis to a later investigation.

\subsubsection{Step 1: Constructing a Smooth Surrogate}
It will be useful for us to arrange terms in our loss function in the following canonical form
\[
\cL(\vtheta) = \frac1{NL}\sum_{i \in [N]}\sum_{l \in [L]: y_{il} = -1} \ell_{il}(\vtheta)
\]
i.e. $\ell_{il}(\vtheta)$ absorbs all dependence on the positive labels of data point $i$. For instance, if using the squared triplet loss function which is indeed smooth, we would have
\[
\ell_{il}(\vtheta) = \frac1{\bar p}\sum_{m \in [L]: y_{im} = 1}([\cE_{\vtheta}(\vz_l)^{\top}\cE_{\vtheta}(\vx_i) - \cE_{\vtheta}(\vz_m)^{\top}\cE_{\vtheta}(\vx_i) + \gamma]_+)^2,
\]
where $\bar p$ is some normalization constant such as the average number of relevant labels per data point. We note that the above expressions have been written from the point of view of module M1 that only trains the embedding model $\cE_{\vtheta}$. However, this is without loss of generality and similar steps can be applied to bring module M2 loss functions into such as canonical form as well. 
Results from Lemma~\ref{lem:smooth} assure us that if the individual loss terms $\ell_{il}$ are $H$-smooth and have $G$-bounded gradients, then so does any composite loss function formed using their sum, as is done while defining $\cL$. Smoothness can be readily ensured by using differentiable activation functions in the neural architecture e.g. GeLU instead of ReLU, as well as using a smooth loss function e.g. using squared triplet loss i.e. $([x - y + \gamma]_+)^2$ or its logistic variant $\ln(1 + \exp(x - y))$ instead of the non-differentiable triplet loss $[x - y + \gamma]_+$. 

Given the above, we define three surrogate functions. The first function $\cL^r(\vtheta)$ only considers terms corresponding to $r$-hard-negatives. However, $\cL^r(\vtheta)$ is not a smooth function and our proofs will require smooth objectives to be used. The second function $\tilde\cL(\vtheta)$ is a smoothed version and pays negligible attention to non-hard negative terms. The third function $\cL^{\text{\alg}}(\vtheta)$ simply considers only those terms that were correctly retrieved by \alg's negative mining strategy.
\begin{align*}
	\cL^r(\vtheta) &:= \frac1{NL}\sum_{i \in [N]}\sum_{l \in [L]: y_{il} = -1} \ell_{il}(\vtheta)\cdot\ind{\norm{\cE_{\vtheta}(\vx_i) - \cE_{\vtheta}(\vz_l)}_2 \leq r}\\
	\tilde\cL(\vtheta) &:= \frac1{NL}\sum_{i \in [N]}\sum_{l \in [L]: y_{il} = -1} \ell_{il}(\vtheta)\cdot d_\Lambda(\norm{\cE_{\vtheta}(\vx_i) - \cE_{\vtheta}(\vz_l)}_2, r)\\
	\cL^{\text{\alg}}(\vtheta) &:= \frac1{NL}\sum_{i \in [N]}\sum_{l \in [L]: y_{il} = -1} \ell_{il}(\vtheta)\cdot\ind{\norm{\cE_{\vtheta}(\vx_i) - \cE_{\vtheta}(\vz_l)}_2 \leq r}\cdot\ind{\neg E^r_{il}},
\end{align*}
where we define the \emph{decay} function as
\[
d_\Lambda(v, r) = \min\bc{\exp(-\Lambda(v - r)),1}
\]
Clearly $d_\Lambda(v, r) = 1$ if $v < r$ and $d_\Lambda(v, r) \rightarrow 0$ rapidly if $v > r$ with the rate of decay being controlled by the \textit{temperature} parameter $\Lambda$. Using the various parts of Lemma~\ref{lem:smooth}, it can be shown that $d_\Lambda(v, r)$ is a smooth function for any value of $r, \Lambda$ and so is $\tilde\cL(\vtheta)$. It is easy to see that given \alg's negative mining strategy that only consider $r$-hard negatives, it performs gradient updates w.r.t $\cL^r(\vtheta)$. However, the following discussion shows that such updates offer bounded bias with respect to the actual gradients for $\tilde\cL(\vtheta)$ as well.

\subsubsection{Step 2: Bounding the Gradient Bias}
The gradient updates actually used by \alg have two sources of bias if the objective $\tilde\cL(\vtheta)$ is considered:
\begin{enumerate}
	\item Terms corresponding to not-so-hard negatives that are never considered by \alg's negative mining strategy i.e. where $\norm{\cE_{\vtheta}(\vx_i) - \cE_{\vtheta}(\vz_l)}_2 > r$ but $d_\Lambda(v, r)$ is not small enough to be negligible.
	\item Terms corresponding to hard negatives missed by \alg's negative mining strategy.
\end{enumerate}
The first source of bias can be handled simply by taking $\Lambda$ to be a large enough quantity. However, we stress that an execution of \alg in practice has to never worry about setting $\Lambda$ appropriately since $\tilde\cL$ is a surrogate created purely for the purpose of establishing the convergence guarantee and is never actually used to perform training.

For the second source of bias, we will appeal to Theorem~\ref{thm:main} which assures us that the number of hard-negatives missed by \alg are not too many. Let $\epsilon_{\text{tot}}^t = c_1\epsilon_1^t + c_2\epsilon_2^t$ denote the total error assured by Theorem~\ref{thm:main-full} in terms of hard-negative terms missed by \alg at iteration $t$. Our goal is to show that the corresponding error in estimating the gradients would be small as well. Now, the absolute error in the gradients can be estimated readily. Since $\ell_{il}(\vtheta)$ has $G$-bounded gradients, a simple application of the triangle inequality allows us to show that
\[
\norm{\nabla\cL^{\text{\alg}}(\vtheta^t) - \nabla\cL^r(\vtheta^t)}_2 \leq G\epsilon_{\text{tot}}^t
\]
However, our subsequent analysis requires the \emph{relative} error in gradients to be small rather than the absolute error. To overcome this problem, fix some $\phi < 1$ and notice that if the absolute error does not translate to a $\phi$-relative error i.e.
\[
G\epsilon^t_{\text{tot}} > \phi\cdot\norm{\nabla\cL^r(\vtheta^t)}_2,
\]
then this implies that
\[
\norm{\nabla\cL^r(\vtheta^t)}_2 \leq \frac{G\epsilon_{\text{tot}}^t}\phi,
\]
which seems to suggest that $\vtheta^t$ is already an approximate first-order stationary point since $\epsilon_{\text{tot}}^t$ is expected to be vanishingly small, but with respect to the effective training objective $\nabla\cL^r(\vtheta)$. Below we show that we can convert this into an approximate stationarity guarantee w.r.t. $\tilde\cL(\vtheta)$, as well as handle cases where we are able to obtained a relative error bound.

\subsubsection{Step 3: Ensuring Convergence to an Approximate First-order Stationary Point}
We are now ready to establish the convergence proof. Lemma~\ref{lem:conv} offers a generic convergence result on a smooth objective when gradient updates are \emph{biased} and may have bounded errors. If $\tilde\cL(\vtheta)$ is taken as the effective training objective (recall that we have already established above that it is smooth), then the bias in gradients offered by \alg can be bounded as
\[
\norm{\nabla\cL^{\text{\alg}}(\vtheta^t) - \nabla\tilde\cL(\vtheta^t)}_2 \leq \norm{\nabla\cL^{\text{\alg}}(\vtheta^t) - \nabla\cL^r(\vtheta^t)}_2 + \norm{\nabla\cL^r(\vtheta^t) - \nabla\tilde\cL(\vtheta^t)}_2
\]
We will set $\Lambda$ to be large enough (recall that this is purely for sake of analysis and does not need to be done in practice) so that the decay function $d_\Lambda(v,r)$ dies so rapidly that we get $\norm{\nabla\cL^r(\vtheta) - \nabla\tilde\cL(\vtheta)}_2 \leq \min\bc{\frac{(1-\phi)}{2(1+\phi)}\norm{\nabla\tilde\cL(\vtheta)}_2, \frac{G\epsilon_{\text{tot}}}\phi}$ for all $\vtheta$.
The first case we consider now is when the absolute gradient error fails to translate to a relative gradient error. As analyzed above, this means that we have
\[
\norm{\nabla\cL^r(\vtheta^t)}_2 \leq \frac{G\epsilon_{\text{tot}}^t}\phi
\]
Applying triangle inequality tells us that
\[
\norm{\nabla\tilde\cL(\vtheta^t)}_2 \leq \frac{G\epsilon_{\text{tot}}^t}\phi + \norm{\nabla\cL^r(\vtheta^t) - \nabla\tilde\cL(\vtheta^t)}_2 \leq \frac{2G\epsilon_{\text{tot}}^t}\phi,
\]
due to the way we set $\Lambda$. In the other case, we do have a relative bound which allows us to bound the first term as
\[
\norm{\nabla\cL^{\text{\alg}}(\vtheta^t) - \nabla\cL^r(\vtheta^t)}_2 \leq \phi\cdot\norm{\nabla\cL^r(\vtheta^t)}_2 \leq \phi\cdot\norm{\nabla\tilde\cL(\vtheta^t)}_2  + \phi\cdot\norm{\nabla\cL^r(\vtheta^t) - \nabla\tilde\cL(\vtheta^t)}_2.
\]
Yet again, due to the way we set $\Lambda$, this gives us
\[
\norm{\nabla\cL^{\text{\alg}}(\vtheta^t) - \nabla\cL^r(\vtheta^t)}_2 \leq \frac{(1-\phi)}2\cdot\norm{\nabla\tilde\cL(\vtheta^t)}_2,
\]
which fulfills the preconditions of Lemma~\ref{lem:conv}. Thus, in every iteration, either we arrive at an approximate first-order stationary point satisfying 
\[
\norm{\nabla\tilde\cL(\vtheta^t)}_2 \leq \frac{2G\epsilon_{\text{tot}}^t}\phi,
\]
or else we continue satisfying the bias conditions of Lemma~\ref{lem:conv}. For small enough step lengths (see Lemma~\ref{lem:conv} for an exact statement), we are ensured that \alg reaches a point $\vtheta^t$ where $\norm{\nabla\tilde\cL(\vtheta^t)}_2 \leq \bigO{\frac1T}$ within $T$ iterations if the bias preconditions keep getting fulfilled in each iteration. This concludes the convergence proof.

\subsection{Supporting Results}
\begin{lemma}[First-order Stationarity with a Smooth Objective]
\label{lem:conv}
Let $f: \vTheta \rightarrow \bR$ be a $H$-smooth objective over model parameters $\theta \in \vTheta$ that is being optimized using a biased gradient oracle and the following update for some step length $\eta > 0$:
\[
\vtheta^{t+1} = \vtheta^t - \eta\cdot\vg^t
\]
Then, if the bias of the oracle is bounded, specifically for some $\delta \in (0,1)$, for all $t$, we have $\vg^t = \nabla f(\vtheta^t) + \vDelta^t$ where $\norm{\vDelta^t}_2 \leq \delta\cdot\norm{\nabla f(\vtheta^t)}_2$ and if the step length satisfies $\eta < \frac{(1-\delta)}{2H(1+\delta^2)}$, then for any $T > 0$, for some $t \leq T$ we must have
\[
\norm{\nabla f(\vtheta^t)}_2^2 \leq \frac{2(f(\vtheta^0) - f(\vtheta^\ast))}{\eta(1-\delta)T}.
\]
\end{lemma}
The above assures that executing the descent step for $\Om T$ steps must result in an iterate at which the square of the gradient norm dips to $\bigO{\frac1T}$ which assures convergence to a first-order stationary point in the limit.
\begin{proof}[Proof (of Lemma~\ref{lem:conv})]
Smoothness of the objective gives us
\[
f(\vtheta^{t+1}) \leq f(\vtheta^t) + \ip{\nabla f(\vtheta^t)}{\vtheta^{t+1} - \vtheta^t} + \frac H2\norm{\vtheta^{t+1} - \vtheta^t}_2^2
\]
Since we used the update $\vtheta^{t+1} = \vtheta^t - \eta\cdot\vg^t$ and we have $\vg^t = \nabla f(\vtheta^t) + \vDelta^t$, the above gives us
\begin{align*}
	f(\vtheta^{t+1}) &\leq f(\vtheta^t) - \eta\cdot\ip{\nabla f(\vtheta^t)}{\vg^t} + \frac{H\eta^2}2\norm{\vg^t}_2^2\\
	&= f(\vtheta^t) - \eta\cdot\ip{\nabla f(\vtheta^t)}{\nabla f(\vtheta^t) + \vDelta^t} + \frac{H\eta^2}2\br{\norm{\nabla f(\vtheta^t) + \vDelta^t}_2^2}\\
	&= f(\vtheta^t) - \eta\cdot\norm{\nabla f(\vtheta^t)}_2^2 - \eta\cdot\ip{\nabla f(\vtheta^t)}{\vDelta^t} + \frac{H\eta^2}2\br{\norm{\nabla f(\vtheta^t) + \vDelta^t}_2^2}
\end{align*}
Now, the Cauchy-Schwartz inequality along with the bound on the bias gives us $- \eta\cdot\ip{\nabla f(\vtheta^t)}{\vDelta^t} \leq \eta\delta\cdot\norm{\nabla f(\vtheta^t)}_2^2$ as well as $\norm{\nabla f(\vtheta^t) + \vDelta^t}_2^2 \leq 2(1+\delta^2)\cdot\norm{\nabla f(\vtheta^t)}_2^2$. Using these gives us
\begin{align*}
	f(\vtheta^{t+1}) &\leq f(\vtheta^t) - \eta(1 - \delta - \eta H(1+\delta^2))\cdot\norm{\nabla f(\vtheta^t)}_2^2\\
									 &\leq f(\vtheta^t) - \frac{\eta(1 - \delta)}2\cdot\norm{\nabla f(\vtheta^t)}_2^2,
\end{align*}
since we chose $\eta < \frac{(1-\delta)}{2H(1+\delta^2)}$. Reorganizing, taking a telescopic sum over all $t$, using $f(\vtheta^{T+1}) \geq f(\vtheta^\ast)$ and making an averaging argument tells us that since we set , for any $T > 0$, it must be the case that for some $t \leq T$, we have
\[
\norm{\nabla f(\vtheta^t)}_2^2 \leq \frac{2(f(\vtheta^0) - f(\vtheta^\ast))}{\eta(1-\delta)T}
\]
\end{proof}
\begin{lemma}
\label{lem:smooth}
Given two bounded non-negative-valued functions $f, g: \cX \rightarrow \bR_+$ that are w.l.o.g. $1$-smooth, $1$-Lipschitz i.e. $\abs{f(\vx) - f(\vy)} \leq \norm{\vx - \vy}_2$ for all $\vx, \vy \in \cX$ (similarly for $g$), $1$-bounded i.e. $f(\vx), g(\vx) \leq 1$ for all $\vx \in \cX$ as well as have $1$-bounded gradient norms i.e. $\norm{\nabla f(\vx)}_2, \norm{\nabla g(\vx)}_2 \leq 1$ for all $\vx \in \cX$. Then for any $c > 0$ the functions $f + g, f \cdot g$ and $\min\bc{f,c}$ are smooth as well. Moreover, the function $\min\bc{f,c}$ is $1$-Lipschitz, $c$-bounded as well as upto isolated points of non-differentiability, has $1$-bounded gradient norms as well. The smoothness, Lipschitz-ness and boundedness constants are taken to be unity here to avoid clutter and the results hold for any positive constant values for these as well.
\end{lemma}
\begin{proof}
We are given that for all $\vx, \vy \in \cX$, we have
\begin{align}
	f(\vx) \leq f(\vy) + \ip{\nabla f(\vy)}{\vx - \vy} + \frac12\norm{\vx - \vy}_2^2 \label{eq:f}\\
	g(\vx) \leq g(\vy) + \ip{\nabla g(\vy)}{\vx - \vy} + \frac12\norm{\vx - \vy}_2^2 \label{eq:g}
\end{align}
We prove the results in parts
\begin{enumerate}
	\item Adding \eqref{eq:f} and \eqref{eq:g} gives us
	\[
	(f+g)(\vx) \leq (f+g)(\vy) + \ip{\nabla(f+g)(\vy)}{\vx - \vy} + \norm{\vx - \vy}_2^2,
	\]
	which establishes that $f+g$ is $2$-smooth.
	\item Multiplying \eqref{eq:f} with $g(\vx)$ gives us
	\[
	f(\vx)g(\vx) \leq f(\vy)g(\vx) + \ip{g(\vy) \nabla f(\vy)}{\vx - \vy} + \ip{(g(\vx) - g(\vy))\nabla f(\vy)}{\vx - \vy} + \frac{g(\vx)}2\norm{\vx - \vy}_2^2
	\]
	Similarly, multiplying \eqref{eq:g} with $f(\vy)$ gives us
	\[
	f(\vy)g(\vx) \leq f(\vy)g(\vy) + \ip{f(\vy)\nabla g(\vy)}{\vx - \vy} + \frac{f(\vy)}2\norm{\vx - \vy}_2^2
	\]
	Adding the two gives us
	\[
	(f\cdot g)(\vx) \leq (f\cdot g)(\vy) + \ip{\nabla(f\cdot g)(\vy)}{\vx - \vy} + \ip{(g(\vx) - g(\vy))\nabla f(\vy)}{\vx - \vy} + \frac{f(\vy) + g(\vx)}2\norm{\vx - \vy}_2^2
	\]
	Lipschitz-ness and the Cauchy-Schwartz inequality tells us that
	\[
	\ip{(g(\vx) - g(\vy))\nabla f(\vy)}{\vx - \vy} \leq \norm{\nabla f(\vy)}_2\norm{\vx - \vy}_2^2 \leq \norm{\vx - \vy}_2^2
	\]
	due to bounded gradients. This and the fact that $f, g$ take bounded values gives us
	\[
	(f\cdot g)(\vx) \leq (f\cdot g)(\vy) + \ip{\nabla(f\cdot g)(\vy)}{\vx - \vy} + 2\norm{\vx - \vy}_2^2
	\]
	which tells us that $f\cdot g$ is $4$-smooth.
	\item Let us define the shortcut $h(\vx) := \min\bc{f(\vx), c}$. Notice that $h$ may have isolated points of non-differentiability where $f(\vx) = c$. We will arbitrarily define $\nabla h(\vx) = \vzero$ at such points. Consider the following 4 cases
	\begin{enumerate}
		\item \textbf{Case 1}: $f(\vx) > c, f(\vy) > c$: in this case $h(\vx) = h(\vy)$ and $\nabla h(\vy) = \vzero$ and thus it is indeed true that $h(\vx) \leq h(\vy) + \ip{\vzero}{\vx - \vy} + \frac12\norm{\vx - \vy}_2^2$
		\item \textbf{Case 2}: $f(\vx) > c, f(\vy) < c$: in this case smoothness of $f$ gives us $f(\vx) \leq f(\vy) + \ip{\nabla f(\vy)}{\vx - \vy} + \frac12\norm{\vx - \vy}_2^2$ but since $f(\vx) > c = h(\vx)$ we get $h(\vx) < f(\vy) + \ip{\nabla f(\vy)}{\vx - \vy} + \frac12\norm{\vx - \vy}_2^2$. However since $f(\vy) = h(\vy), \nabla f(\vy) = \nabla h(\vy)$, we get $h(\vx) \leq h(\vy) + \ip{\nabla h(\vy)}{\vx - \vy} + \frac12\norm{\vx - \vy}_2^2$.
		\item \textbf{Case 3}: $f(\vx) < c, f(\vy) > c$: in this case $\nabla f(\vy) = \vzero$ but we still have $h(\vx) = f(\vx) < c = h(\vy)$ as well as $h(\vy) \leq h(\vy) + \ip{\vzero}{\vx - \vy} + \frac12\norm{\vx - \vy}_2^2$ which tells us that $h(\vx) \leq h(\vy) + \ip{\vzero}{\vx - \vy} + \frac12\norm{\vx - \vy}_2^2$
		\item \textbf{Case 4}: $f(\vx) < c, f(\vy) < c$: in this case $h(\vx) = f(\vx), \nabla h(\vy) = \nabla f(\vy), h(\vy) = f(\vy)$ and $h$ simply inherits the smoothness of $f$ to give $h(\vx) \leq h(\vy) + \ip{\nabla h(\vy)}{\vx - \vy} + \frac12\norm{\vx - \vy}_2^2$
	\end{enumerate}
	The above cases establish that $h = \min\bc{f, c}$ is $1$ smooth as well. It is easy to see that the function $h$ is $1$-Lipschitz as well as $c$-bounded. Since at all points of non-differentiabilty, we defined the gradient as the zero vector, the function has $1$-bounded gradient norms as well.
\end{enumerate}
\end{proof}

\end{document}